\newif\ifdraft\draftfalse
\documentclass{article}

\usepackage{arxiv}

\usepackage[utf8]{inputenc}
\usepackage[T1]{fontenc}
\usepackage{hyperref}
\usepackage{graphicx}
\usepackage{natbib}
\usepackage{caption}
\usepackage{subcaption}
\usepackage{algorithm}
\usepackage{algorithmic}
\usepackage{booktabs}
\usepackage{amsmath,amsthm,amssymb}
\usepackage{amsfonts}
\usepackage{microtype}
\usepackage{comment}
\usepackage{enumitem}
\usepackage{multirow}
\usepackage{multicol}
\usepackage{tikz}
\usepackage{xcolor}
\usepackage{textgreek}
\usepackage{tcolorbox}
\usepackage{overpic}

\newtheorem{thm}{Theorem}
\newtheorem{lemma}{Lemma}
\newtheorem{definition}{Definition}

\newtheorem{proposition}{Proposition}
\newtheorem{remark}{Remark}
\newtheorem{assumpt}{Assumption}

\definecolor{deepgreen}{RGB}{0,100,0}
\newcommand\inv[1]{#1\raisebox{1.15ex}{$\scriptscriptstyle-\!1$}}

\newcommand{\emd}{\mathrm{EMD}}
\newcommand{\EE}{\mathbb{E}}

\newcommand{\Ac}{\mathcal{A}}

\newcommand{\Dc}{\mathcal{D}}
\newcommand{\Ec}{\mathcal{E}}

\newcommand{\Gc}{\mathcal{G}}

\newcommand{\Vc}{\mathcal{V}}
\newcommand{\Wc}{\mathcal{W}}

\renewcommand{\Pr}{\mathscr{P}}

\newcommand{\av}{{\bf a}}

\newcommand{\gv}{{\bf g}}

\newcommand{\wv}{{\bf w}}
\newcommand{\xv}{{\bf x}}

\newcommand{\Av}{{\bf A}}

\DeclareMathOperator\E{E}

\let\Pr\relax
\DeclareMathOperator\Pr{P}

\def\textiid{i.i.d.\@\xspace}
\newcommand\iid{\ifmmode\text{ i.i.d. } \else \textiid \fi}

\newcommand{\beqs}{\begin{equation*}}
\newcommand{\eeqs}{\end{equation*}}
\newcommand{\beq}{\begin{equation}}
\newcommand{\eeq}{\end{equation}}

\newcommand{\myhide}[1]{}

\renewcommand{\c}{c}

\title{The Role of Causality in Algorithmic Recourse\thanks{Authors are listed alphabetically.}}

\author{
  Srikanth Avasarala \\
  Georgia Institute of Technology \\
  \texttt{savasarala9@gatech.edu} \\
  \And
  Varun Gupta \\
  Vector Institute \\
   \texttt{varun.gupta@vectorinstitute.ai} \\
  \And
  Shahin Jabbari \\
  Drexel University \\
   \texttt{shahin@drexel.edu} \\
  \And
  Saber Salehkaleybar \\
  Leiden University \\
   \texttt{s.salehkaleybar@liacs.leidenuniv.nl} \\
  \And
  Juba Ziani \\
  Georgia Institute of Technology \\
  \texttt{jziani3@gatech.edu}
}

\date{}
\renewcommand{\shorttitle}{The Role of Causality in Algorithmic Recourse}

\hypersetup{
  pdftitle={The Role of Causality in Algorithmic Recourse},
  pdfauthor={Srikanth Avasarala, Varun Gupta, Shahin Jabbari, Saber t, Juba Ziani},
  pdfkeywords={algorithmic recourse, causality, performative prediction, strategic behavior}
}

\begin{document}
\maketitle

\begin{abstract}
Algorithmic recourse aims to provide individuals with actionable changes to improve their predicted outcomes in high-stakes classification settings, such as loan and mortgage applications. However, most existing approaches do not consider whether proposed feature changes lead to actual improvement to an agent's true label or qualifications (such as their ability to repay loans on time) as opposed to gaming the system to obtain a better outcome without truly improving the label (e.g., by opening dummy credit card accounts to increase one's credit score). As a result, deployed recourse policies may incentivize agents to game the deployed model, leading to strategic behavior that degrades predictive accuracy and ultimately renders the recourse itself ineffective after model retraining. In this work, we formalize this failure mode through a causal and performative framework for recourse. We model how recourse-induced actions propagate through a structural causal model and explicitly model how feature modifications affect one another and the true label. Such causal responses lead to a non-convex optimization problem, even under standard and convex losses. To address this, we characterize conditions under which performatively stable solutions exist and can be efficiently computed via simple dynamics. Our analysis highlights a key tradeoff: recourse policies that ignore causal structure can induce large, misaligned behavioral responses, whereas causal recourse leads to equilibria that lower the impact of gaming. We empirically demonstrate, across both semi-synthetic and real credit datasets, that accounting for causal responses yields models that reduce incentives to game or manipulate non-causal features.
\end{abstract}

\section{Introduction}
\label{sec:intro}

Machine learning models are deployed in high-stakes decision-making settings, including lending, hiring, insurance, healthcare, and criminal justice, where their predictions directly affect individuals' opportunities and access to resources. As these models have grown increasingly opaque and complex, there has been growing interest in making their predictions more transparent and explainable~\citep{Doshi-VelezK17}. In particular, individuals affected by algorithmic decisions increasingly expect explanations not only for why a prediction was made, but also for \emph{how they might improve their outcomes}. As a result, a major direction in explainable machine learning is \emph{algorithmic recourse}---the goal of which is to provide actionable recommendations to individuals receiving an unfavorable prediction, suggesting changes to their features that would alter the model's decision \citep{WachterMR18,UstunSL19,PawelczykDHKL23}. 

However, existing recourse methods typically optimize only for changing the \emph{prediction} of the deployed model, without considering whether the recommended actions are actually beneficial. They treat the deployed model as fixed and ignore whether the suggested actions invalidate the deployed model. For example, if these modifications constitute strategic gaming, then they improve an agent's outcome with respect to the deployed model without improving their \emph{true} qualifications. Generally, Goodhart's law implies that recourse can create a misalignment between the true and the predicted label; then, deployed models become progressively less accurate, learners must repeatedly redesign their predictive models, and recourse recommendations often become ineffective moving goalposts. 

This work addresses this challenge by explicitly incorporating causality into algorithmic recourse to 
understand how interventions change not only agent features but also their true labels. We ask: \textbf{how can we jointly design a model and a recourse policy that simultaneously promote meaningful improvement while maintaining high predictive accuracy?} Building on tools from performative prediction, we characterize conditions under which stable and optimal solutions exist and can be computed efficiently, and we demonstrate both theoretically and empirically that causal recourse substantially reduces incentives for gaming while significantly improving predictive performance.

\paragraph{Summary of Contributions.} Our contributions are:

\begin{itemize}
    \item In Section~\ref{sec:model}, we introduce a causal performative recourse framework that models the interaction between the learner and strategic agents through a structural causal model, capturing the effects of recourse interventions.

    \item In Section~\ref{sec:alg-full-info}, we show that the resulting learning problem can be formulated as a performative prediction problem. We leverage this performative prediction framework to provide algorithms to obtain causal recourse with formal convergence guarantees.

   \item In Section~\ref{sec:exp}, we empirically evaluate the proposed framework on semi-synthetic and real-world datasets. We demonstrate the convergence of the proposed algorithms, quantify the benefits of stable recourse over naive policies, and examine the impact of causality and the intervention cost structure on the computed recourse.
\end{itemize}
\subsection{Related Work}
\label{sec:related}

\paragraph{Algorithmic Recourse.}
Algorithmic recourse is a post-hoc counterfactual explanation that provides minimal-cost modifications needed to change a model's prediction for an input with an undesirable label~\citep{WachterMR18, UstunSL19}. Several formulations and variants have been proposed~\citep{LooverenK21, PawelczykDHKL23, GargNS25, KarimiBBV20, KarimiSV21, PoyiadziSSBF20, NooraniDHD25, KennyAB+26, Russell19}. For example,~\citet{WachterMR18} consider score-based classifiers and develop feature modifications to help instances reach a target score. Meanwhile,~\citet{UstunSL19} focus on binary classifiers and require recourse actions that alter the predicted label. We focus on a regression setting where higher scores are desirable, and our recourse formulation aims to maximize this score. See~\citep{VermaDH20, KarimiBSV23} for surveys.

\paragraph{Robust Recourse.}
The robust recourse literature ensures that following recourse leads to desirable outcomes even when the underlying predictive model changes slightly~\citep{UpadhyayJL21}. Similar to algorithmic recourse, various formulations and extensions of robust recourse have been explored~\citep{YetukuriHVUL24, JiangLR+24a, NguyenBN23, DuttaLMTM22, MochaourabSG+22, CheonWF+25, HammanNMMD23, BewleyAM+24, GuoJC+23, KyawKJ25, KayasthaGJ26}. See~\citep{JiangLR+24a} for a survey.

\paragraph{Strategic Classification.} 
A closely related area to algorithmic recourse is strategic classification, where individuals (or agents) manipulate or
modify their features to improve their outcomes~\citep{HardtMP+16, BravermanG20, DongRS+18, SundaramVX+23}. Despite similarities in formulations, the main premise in strategic classification is to design decision-making strategies that are robust to manipulation (or gaming), while the recourse focuses on actionable instructions on how to legitimately change the features to reverse an unfavorable decision. 

\paragraph{Performative Prediction.}
Related to our work and one of the main tools we leverage in this paper is the idea of ``performative prediction''. In performative prediction, model predictions influence behavior and thereby shift the data distribution~\citep{PerdomoZMH20, KimP23, PerdomoBH+25, Mendler-DunnerDW22, HardtM23}; this is similar to our setting where recommended actions also shift agent behavior. \citet{PerdomoZMH20} show that standard empirical risk minimization can fail under such feedback and distinguish between performative optimality, which minimizes loss on the induced distribution, and performative stability, where model deployment reproduces the same model. They propose algorithms and conditions for computing stable points and show that optimal models lie near stable ones. Follow-up work provided algorithms for computing performatively optimal models under parametric assumptions on data distributions~\citep{IzzoYZ21, IzzoZY22}. Our formulation corresponds to computing performatively optimal solutions, and our main techniques rely on tools from performative prediction.

\paragraph{The Role of Improvement and Causality.}
Tools from causality have been used to account for interdependence between features~\citep{ShavitEA20, HaghtalabIL+20, GoisGR+26}
and to incentivize actions that focus on improvement rather than manipulation~\citep{KanamoriKH+25,EfthymiouFP25, EfthymiouPS+25, KleinbergR19, BechavodLW+21}. \citet{KonigFG23} propose improvement-focused causal recourse, leveraging causal knowledge to compute interventions that guarantee improvement in the underlying state. 
Finally, through a causality lens,~\citet{KonigFF+25} study conditions under which recourse remains valid after a distribution change. Unlike these works, we study the interaction between recourse interventions and model retraining, formulating causal recourse as a performative prediction problem. 

\section{Model}
\label{sec:model}
We study a recourse setting in which a learner deploys a scoring rule and provides agents with actionable recommendations to improve their predicted outcomes. Agents respond by taking these actions that modify their features and potentially their underlying qualifications. These strategic responses induce a model-dependent distribution shift, and the learner’s goal is to optimize the deployed model while accounting for the resulting behavioral adaptations. We next provide a formal description of the model. 

\paragraph{Learner-Agent Model}
We denote $\mathcal X\subseteq\mathbb{R}^d$ the feature space and
$\mathcal Y\subseteq\mathbb{R}$ the space of real-valued scores; larger values correspond to better outcomes. We assume an underlying
distribution $\mathcal D$ over $\mathcal X\times\mathcal Y$. The learner
deploys a linear predictor from hypothesis class $\mathcal H
=
\{h_{\wv}(\xv)=\wv^\top\xv
\mid
\wv\in\mathcal W\}$,
with $\mathcal W\subseteq\mathbb{R}^d$ compact.

Following the deployment of the model, the learner recommends a recourse action for each agent. The recommended action specifies interventions on the agent's features that are intended to improve the prediction produced by the deployed model. We model a recourse action as a vector
$\av\in\mathcal A\subseteq\mathbb{R}^d$,
where each coordinate corresponds to one of the observable features. E.g., in a hiring setting, recourse actions may correspond to acquiring additional qualifications or work experience, whereas in a lending setting they may correspond to reducing debt or increasing savings.
 Since interventions on one feature can propagate and influence other features, it is crucial to model the \emph{causal} relationships between features. 

\paragraph{Causal Model of Features.}
We model causal dependencies using a weighted directed graph
$\Gc=(\Vc,\Ec,\omega)$,
where $\Vc$ denotes the set of features, $\Ec$ the set of directed causal dependencies, and $\omega:\Ec\rightarrow\mathbb{R}$ assigns edge weights representing causal strengths. We assume that $\Gc$ is a directed acyclic graph (DAG), as is common in the related literature~\citep{EfthymiouPS+25,karimi2020algorithmic}. The causal structure is represented by an adjacency matrix $\mathbf A\in\mathbb R^{d\times d}$, where $A_{ij}=\omega(i,j)$ if $(i,j)\in\Ec$, and $A_{ij}=0$ otherwise.

An intervention on one feature may induce both direct changes and indirect changes through downstream causal pathways. To capture these cumulative effects, we define the \emph{feature contribution matrix} $G_x$, which maps an intervention to the resulting feature changes after accounting for all causal propagation through the graph. The resulting feature and label changes are then given by
$\Delta\xv(\av)=G_x^\top\av,
~\text{and}
~\Delta y(\av)=G_y^\top\av$. Here, the feature changes can be interpreted as arising from additive interventions on an underlying structural causal model (SCM); details and computation of $G_x$ from the causal graph are provided in Appendix~\ref{app:scm}. 

\paragraph{Agent's Response.}
After deploying a scoring rule $h_{\wv}$, the learner recommends a recourse action to each agent. We model this recommendation as the utility-maximizing intervention $\av\in\Ac\subseteq\mathbb{R}^d$, with cost
$c_\kappa(\av)=\kappa^{-1}c(\av)$,
where $c:\mathbb{R}^d\to\mathbb{R}_{\geq 0}$ and $\kappa>0$ controls the intervention strength. The recommended action is 
\begin{equation}
\av^*(\xv,\wv)
\in
\arg\max_{\av\in\Ac}
\left\{
\wv^\top\bigl(\xv+\Delta\xv(\av)\bigr)
-
\frac{1}{\kappa}c(\av)
\right\}.
\label{eq:inner}
\end{equation}
Writing $\av^*$ when dependencies are clear, the resulting feature and score values are $\xv+G_x^\top\av^*,~y+G_y^\top\av^*$. Appendix~\ref{app:icr} presents an alternative formulation that minimizes intervention cost subject to a prescribed prediction improvement.

\paragraph{Learner's Optimization Problem.}
Implementing the recommended recourse induces a predictor-dependent distribution $\Dc_\wv$. The learner therefore seeks a predictor that minimizes the expected loss under the induced distribution:
\begin{equation}
 \wv^*
\in
\arg\min_{\wv\in\Wc}
\quad
\EE_{(\xv,y)\sim\Dc_{\wv}}
\!\left[
\ell(\wv^\top\xv,y)
\right].
\label{eq:outer-exp}
\end{equation}

In practice, the learner observes only samples
$S=\{(\xv_i,y_i)\}_{i=1}^m$ drawn from $\Dc$. Applying the recommended recourse to each sample produces the induced dataset
$S'
=
\left\{
(\xv_i+\Delta\xv(\av^*),\,
y_i+\Delta y(\av^*))
\right\}_{i=1}^m$,
which approximates $\Dc_\wv$. The resulting empirical optimization problem is
\begin{align}
\wv^*
\in
\arg\min_{\wv\in\Wc}
\frac{1}{m}
\sum_{i=1}^m
\ell\Big(
\wv^\top(&\xv_i+\Delta\xv(\av^*)),  y_i+\Delta y(\av^*)
\Big).
\label{eq:outer-sample}
\end{align}
\section{Algorithm and Analysis}
\label{sec:alg-full-info}

Our goal in this section is to solve the optimization problem in Equation~\ref{eq:outer-sample}. We first state the assumptions used throughout that make the framework amenable to analysis. All omitted proofs are provided in Appendix~\ref{app:proofs}.

\begin{assumpt}
\label{assumpt:ell}
The loss function $\ell(\xv,y)$ is jointly convex in both of its arguments.
\end{assumpt}

\begin{assumpt}
\label{assumpt:cost}
The cost function $\c$ is assumed to be \emph{quadratic}, i.e.,
$\c(\av)=\av^{\top}C\av$, where $C\in\mathbb{R}^{n\times n}$ is a positive-definite cost matrix. This choice is standard in the strategic classification and algorithmic recourse literature \citep{BechavodPWZ22,AvasaralaWZ25,EfthymiouFP25}.
\end{assumpt}

\paragraph{Characterization of the Agents' Best Response} First, we show that the optimization problem in Equation~\ref{eq:outer-sample} can be non-convex. 
We first observe that under Assumption~\ref{assumpt:cost}, there is a closed-form solution for the recourse actions.
\begin{lemma}
\label{pro:recourse-closed-form}
Under Assumption~\ref{assumpt:cost}, the optimal recourse action admits the closed-form expression,
$\av^*
=
\frac{\kappa}{2}
C^{-1}G_x\wv$.
\end{lemma}
In turn, Equation~\ref{eq:outer-sample} becomes:
\begin{align}
\wv^*
\in
\arg\min_{\wv\in\Wc}
\quad
\frac{1}{m}\sum_{i=1}^m
\ell\Big(
&
\wv^\top
\Big(
\xv_i
+
\frac{\kappa}{2}
G_x^\top C^{-1}G_x\wv
\Big),
y_i
+
\frac{\kappa}{2}
G_y^\top C^{-1}G_x\wv
\Big).
\label{eq:outer-sample-simplified}
\end{align}

This optimization problem can be difficult to solve: 
\begin{lemma}
\label{pro:non-convex}    
The optimization problem in Equation~\ref{eq:outer-sample-simplified} can be non-convex even if $\ell$ and $\c$ satisfy Assumptions~\ref{assumpt:ell}~and~\ref{assumpt:cost}.
\end{lemma}
I.e., standard convex optimization techniques do not directly apply. Instead, we leverage the performative prediction framework~\citep{PerdomoZMH20} to develop iterative algorithms that account for the distribution shift induced by agents' recourse responses.

\paragraph{Performative Formulation.}
The deployed predictor determines the agents' recourse actions and therefore induces a predictor-dependent data distribution. Consequently, our learning problem is an instance of performative prediction~\citep{PerdomoZMH20}. The optimizer of Equation~\ref{eq:outer-exp} is referred to as a \emph{performatively optimal} solution, which minimizes the performative risk over its induced distribution. Since computing performative optima is generally challenging, we instead seek \emph{performatively stable} solutions, which are fixed points of the retraining dynamics, i.e., $\wv^s$ is performatively stable if deploying model $\wv^s$ and retraining on the distribution induced by the deployed predictor yields the exact same predictor $\wv^s$. Formally, it is defined as
\begin{equation}
 \wv^s\in \arg\min_{\wv\in\Wc} \quad \E_{(\xv,y)\sim\Dc_{\wv^s}} \left[\ell\left(\wv^T \xv, y\right) \right].
\label{eq:outer-exp-stable}
\end{equation}
For the specific form of feature and score changes induced by Lemma~\ref{pro:recourse-closed-form}, the corresponding finite-sample optimization problem for a performatively stable solution is
\begin{align}
\wv^s
\in
\arg\min_{\wv\in\Wc}
\quad
\frac{1}{m}\sum_{i=1}^m
\ell\Big(
&
\wv^\top
\Big(
\xv_i
+
\frac{\kappa}{2}
G_x^\top C^{-1}G_x\wv^s
\Big),
y_i
+
\frac{\kappa}{2}
G_y^\top C^{-1}G_x\wv^s
\Big).
\label{eq:outer-sample-simplified-stable}
\end{align}

To prove convergence guarantees, we impose the
following standard regularity assumptions on the loss function $\ell$.
These assumptions are satisfied by several standard loss functions
used in practice, including the mean squared error (MSE) loss and
logistic regression loss under bounded feature norms.

\begin{assumpt}
\label{assumpt:ell-strong-convex}
$\ell$ is $\gamma$-strongly convex in $\wv$ $\forall (\xv,y)$. 
\end{assumpt}

\begin{assumpt}
\label{assumpt:ell-smooth}
$\ell$ is $\beta$-jointly smooth\footnote{A function
is $\beta$-jointly smooth if its gradient is $\beta$-Lipschitz with
respect to its arguments.} in $\wv$ and $(\xv,y)$.
\end{assumpt}

\subsection*{Computing Stable and Near-optimal Solutions}
We now present iterative algorithms for computing performatively stable decision models under the induced performative recourse problem, and establish conditions under which the resulting stable solutions are close to performatively optimal.

\paragraph{Repeated Risk Minimization.}
Algorithm~\ref{alg:rrm} is a finite-sample adaptation of Repeated Risk Minimization (RRM)~\citep{PerdomoZMH20} augmented with our causal recourse framework. It incorporates the causal response model through induced feature and score shifts governed by $G_x,~G_y$.
\begin{algorithm}[ht!]
\begin{algorithmic}[1] 
\STATE \textbf{Input}: C, $G_x$, $G_y$, $\lambda$, number of rounds $T$, and sample sizes $n$ for $t\in [T]$   \\
\STATE \textbf{Output}: Weight vector $\wv^T$
\STATE $\wv^{0} \gets $ a random vector from $\Wc$\\
\STATE $A \gets \frac{\kappa}{2} G_x^\top C^{-1} G_x$
\STATE $B \gets \frac{\kappa}{2} G_y^\top C^{-1} G_x$
\FOR{$t = 1, \ldots, T$}
{
	\STATE $S\gets$ a sample of size $n$ drawn from $D$\\
	\STATE Compute $S'$ by changing $(\xv,y) \in S$ to $(\xv',y')=(\xv+A\wv^{t-1}, y+B\wv^{t-1})\in S'$  \\
	\STATE $\wv^{t}\gets \arg\min_{w\in\Wc} \Sigma_{(\xv',y')\in S'} \ell\left(\wv^T \xv_i', y_i'\right)$\\
 }
 \ENDFOR
\RETURN $\wv^T$
\end{algorithmic}
\caption{Repeated Risk Minimization (RRM)}\label{alg:rrm}
\end{algorithm}

We next present \emph{Repeated Gradient Descent} (RGD)~\citep{PerdomoZMH20}  in Algorithm~\ref{alg:rgd}. RGD is a gradient-based approximation to RRM---instead of solving the empirical risk minimization problem exactly at each iteration, we perform a gradient descent update on the empirical objective. 
\begin{algorithm}[ht!]
\begin{algorithmic}[1]
\STATE \textbf{Input}: $C$, $G_x$, $G_y$, $\lambda$, step size $\eta$, Euclidean projection operator onto $\Wc$:  $\Pi_{\Wc}(\cdot)$,  number of rounds $T$, and sample sizes $n$ for $t\in[T]$
\STATE \textbf{Output}: Weight vector $\wv^T$
\STATE $\wv^0 \gets$ a random vector from $\Wc$ 
\STATE $A \gets \frac{\kappa}{2} G_x^\top C^{-1} G_x$
\STATE $B \gets \frac{\kappa}{2} G_y^\top C^{-1} G_x$
\FOR{$t = 1,\ldots,T$}
    \STATE $S \gets$ a sample of size $n$ drawn from $\Dc$
    \STATE Compute $S'$ by mapping each $(\xv,y)\in S$ to    
$    (\xv',y')
    =
    \left(
    \xv + A\wv^{t-1},
    y + B\wv^{t-1}
    \right) \in S'$  
    \STATE Compute the empirical gradient   
$    \gv_t
    =
    \nabla_{\wv}
    \left[
    \frac{1}{n_t}
    \sum_{(\xv',y')\in S'}
    \ell\left(\wv^\top \xv',y'\right)
    \right]_{\wv=\wv^{t-1}}$
    
    \STATE $\wv^t \gets \Pi_{\Wc}\left(\wv^{t-1}-\eta \gv_t\right)$
\ENDFOR
\RETURN $\wv^T$
\caption{Repeated Gradient Descent (RGD)}\label{alg:rgd}
\end{algorithmic}
\end{algorithm}

\paragraph{Convergence to a Stable Solution.}
We now establish sufficient conditions under which Algorithms~\ref{alg:rrm} and~\ref{alg:rgd} converge to a performatively stable solution. 

\begin{thm}
\label{pro:rrm-stable}
Suppose Assumptions~\ref{assumpt:cost},~\ref{assumpt:ell-strong-convex}, ~\ref{assumpt:ell-smooth} hold. Additionally, if
there exist constants $\alpha>0$ and $\mu>0$ such that
$\xi_{\alpha, \mu}
=
\int_{\mathbb{R}^d}
e^{\mu{\Vert \xv+A\wv \Vert}_2^\alpha}
d(\Dc)
< \infty$
for all $\wv \in \mathcal{W}$, then for
$n_t = O\left(\tfrac{1}{\delta^d}\log\left(\tfrac{t}{p}\right)\right)$
samples for all $t\in[T]$, and with
$\Lambda(C,G_x,G_y)
:=
\|G_x\|_2^2\|C^{-1}\|_2
+
\|G_x\|_2\|C^{-1}\|_2\|G_y\|_2$:

\begin{enumerate}
    \item if $\kappa < \frac{\gamma}{\beta\Lambda}$,
    then with probability at least $1-p$, Algorithm~\ref{alg:rrm}
    will return a vector $\wv^T$ such that
    $\|\wv^T-\wv^s\|_2 \leq \delta$ for all
    $T\geq\tfrac{\log\left(\tfrac{1}{\delta}\|\wv^0-\wv^s\|_2\right)}
    {1-\tfrac{\kappa\beta\Lambda}{\gamma}}$.

    \item if $\kappa < \frac{2\gamma}
    {\Lambda(\beta+\gamma)(1+1.5\eta\beta)}$,
    then with probability at least $1-p$, Algorithm~\ref{alg:rgd}
    will return a vector $\wv^T$ such that
    $\|\wv^T-\wv^s\|_2\leq\delta$ for all
    $T \geq
    \frac{\log \left(\tfrac{1}{\delta}\|\wv^0-\wv^s\|_2\right)}
    {\eta \left(
    \frac{\beta \gamma}{\beta + \gamma}
    -
    \kappa\Lambda
    \left(\tfrac{3}{2}\eta\beta^2+\beta\right)
    \right)}$.
\end{enumerate}
\end{thm}

\begin{proof}[Proof Sketch]
The proof follows from the convergence analysis of~\citet{PerdomoZMH20}. The main technical step in this paper is to provide a characterization of the sensitivity of the induced response distribution map in our causal recourse setting (see Appendix~\ref{app:sens}). The proofs are deferred to Appendix~\ref{app:proofs}.
\end{proof}

\begin{remark}
If we assume that the initial data distribution $\Dc_0$ has a finite exponential $\alpha$-moment, i.e.,
$\E_{(\xv_0,y_0)\sim\Dc_0}
\left[
\exp\!\left(\bar{\mu}\|\xv_0\|_2^\alpha\right)
\right]
< \infty,$ then the quantity $\xi_{\alpha,\mu}$ is finite.
\end{remark}

The above exponential moment condition is an alternative, distributional assumption that guarantees $\xi_{\alpha,\mu}<\infty$.

\paragraph{Connections to Optimal Solutions.} 
Finally, we characterize convergence to performatively optimal solutions.
\begin{thm}\label{thm:s-0-dist}
Suppose Assumptions~\ref{assumpt:cost}, \ref{assumpt:ell-strong-convex}, and
\ref{assumpt:ell-smooth} hold. Furthermore, assume that the loss function
$\ell(u,v)$ is $L_u$-Lipschitz in its first argument. Finally, suppose there
exist constants $\alpha>0$ and $\mu>0$ such that
$\xi_{\alpha,\mu}
=
\int_{\mathbb{R}^d}
e^{\mu\|\xv+A\wv\|_2^\alpha}
\,d(\Dc)
<\infty,
\qquad \forall\,\wv\in\mathcal{W}.$
Let $\wv^*$ denote a performative optimum.
Then, with $\Lambda(C,G_x,G_y)
:=
\|G_x\|_2^2\|C^{-1}\|_2
+
\|G_x\|_2\|C^{-1}\|_2\|G_y\|_2$, the following holds.

\begin{enumerate}
    \item If
    $
    \kappa < \frac{\gamma}{\beta\Lambda},
    $
     then with probability at least $1-p$, Algorithm~1 returns a vector $\wv^T$ such that    
    $\|\wv^T-\wv^*\|_2
    \le
    \delta+\frac{\kappa L_u\Lambda}{\gamma}$
    for all
    $
    T \ge
    \frac{
    \log\left(\frac{1}{\delta}\|\wv^0-\wv^s\|_2\right)
    }{
    1-\frac{\kappa\beta\Lambda}{\gamma}
    }.
    $

    \item If
    $
    \kappa
    <
    \frac{2\gamma}{\Lambda(\beta+\gamma)(1+1.5\eta\beta)},
    $
    then with probability at least $1-p$, Algorithm~3 returns a vector $\wv^T$ such that   
$    \|\wv^T-\wv^*\|_2
    \le
    \delta+\frac{\kappa L_u\Lambda}{\gamma}$
    for all
    $
    T \ge
    \frac{
    \log\left(\frac{1}{\delta}\|\wv^0-\wv^s\|_2\right)
    }{
    \eta\left(
    \frac{\beta\gamma}{\beta+\gamma}
    -
    \kappa\Lambda\left(\frac{3}{2}\eta\beta^2+\beta\right)
    \right)
    }.
    $
\end{enumerate}
\end{thm}

The convergence conditions reveal the role of the learner--agent feedback. While optimization parameters $\gamma$ and $\beta$ are inherited from performative prediction, our characterization shows that the feedback strength is governed by $\kappa\Lambda$.
Stronger interventions (larger $\kappa$) or greater causal propagation (larger $\Lambda$) amplify the induced distribution shift, leading to more restrictive sufficient conditions for convergence.

\section{Numerical Experiments}
\label{sec:exp}
We now perform numerical experiments to study the convergence properties of RRM and RGD in our framework. We consider two datasets for bank loan approval: (i) a synthetic $7$-dimensional dataset generated from a structural causal model, and (ii) the empirical Taiwan credit dataset. 

\subsection{Datasets}

\paragraph{7D Semi-Synthetic Loan Approval Dataset.}
We evaluate our method on a 7-dimensional semi-synthetic dataset constructed using the same structural causal model as \citep{karimi2020algorithmic}, while adapting the labels to a regression setting. The feature vector is given by
$x = (G,A,E,L,D,I,S),$
where $G$ denotes gender, $A$ age, $E$ education level, $L$ loan amount, $D$ loan duration, $I$ income, and $S$ savings. 
Using standardized variables, we estimate signed edge weights by regressing each child node on its parents, and subsequently construct the causal contribution matrices $G_x$ as described in Section~\ref{sec:model}. We construct $G_y$ by regressing features on the output scores. See Appendix~\ref{app:synth} for the causal graph (Figure~\ref{fig:graph-syn}) and full graph construction with edge weights.

\paragraph{Taiwan Credit Dataset.}
We also evaluate our methods on the \emph{Default of Credit Card Clients} dataset from the UCI Machine Learning Repository \citep{default_of_credit_card_clients_350}. The dataset contains records of $30{,}000$ credit card clients in Taiwan, including demographic information, repayment history, bill statements, and payment amounts collected over six months.

The original dataset contains 23 input variables. To obtain a lower-dimensional representation suitable for our experiments, we aggregate related variables into the following 6 features: $\texttt{PAY\_HISTORY}$, $\texttt{BILL\_PAY}$, $\texttt{AGE}$, $\texttt{SEX\_MAR}$,  $\texttt{EDU}$, and $\texttt{LIMIT\_BAL}$.  The output scores are generated for the labels by scores from logistic regression on the feature scores.
As in the synthetic setting, signed edge weights are estimated via regression on standardized variables, and the feature contribution matrices $G_x$ and $G_y$ are constructed the same way as done for the semi-synthetic data.
Details regarding feature construction and the causal graph (Figure~\ref{fig:graph-tai}) are provided in Appendix~\ref{app:taiwan}.

\subsection{Experimental Design}
Our experiments analyze how strategic adaptation and causal structure influence performative effects. We use the Mean Squared Error (MSE) with $\ell_2$ regularization as our loss. 

Table~\ref{tab:cost_profiles} summarizes the intervention cost profiles considered in this work. In all profiles, the immutable features \textsc{A} and \textsc{G} in the semi-synthetic data, and \textsc{AGE} and \textsc{SEX\_MAR} in the Taiwan data, are assigned infinite intervention costs and are thus not actionable. Unless otherwise specified, we use the \emph{default} profile; results for the remaining profiles are deferred to Appendix~\ref{app:sensitiv}, where we show that the conclusions are robust across intervention cost structures. All reported quantities are averaged over 20 independently generated cost realizations, where each realization is obtained by independently perturbing the finite diagonal entries of the corresponding cost profile uniformly within a fixed radius about the profile center. Error bars denote one standard deviation.

\begin{table*}[t]
\centering
\small

\begin{minipage}[t]{0.48\textwidth}
\centering
\caption*{\textbf{Semi-synthetic}}
\renewcommand{\arraystretch}{1.15}
\begin{tabular}{lccccccc}
\toprule
Cost profile & G & A & E & L & D & I & S \\
\midrule
\textbf{Default}          & $\infty$ & $\infty$ & 3.0 & 2.0 & 2.0 & 1.5 & 1.0 \\
Uniform                   & $\infty$ & $\infty$ & 1.0 & 1.0 & 1.0 & 1.0 & $\infty$ \\
Upstream cheap            & $\infty$ & $\infty$ & 0.6 & 0.8 & 1.0 & 1.4 & 2.2 \\
Downstream cheap          & $\infty$ & $\infty$ & 2.2 & 1.4 & 1.0 & 0.8 & 0.6 \\
\bottomrule
\end{tabular}
\end{minipage}
\hfill
\begin{minipage}[t]{0.48\textwidth}
\centering
\caption*{\textbf{Taiwan}}
\renewcommand{\arraystretch}{1.15}
\begin{tabular}{lcccccc}
\toprule
Cost profile & PH & BP & A & SM & E & LB \\
\midrule
\textbf{Default}          & 1.0 & 1.0 & $\infty$ & $\infty$ & 2.0 & 0.1 \\
Uniform                   & 1.0 & 1.0 & $\infty$ & $\infty$ & 1.0 & $\infty$ \\
Upstream cheap            & 0.8 & 1.0 & $\infty$ & $\infty$ & 1.5 & 2.5 \\
Downstream cheap          & 2.5 & 1.5 & $\infty$ & $\infty$ & 1.0 & 0.8 \\
\bottomrule
\end{tabular}
\end{minipage}

\caption{Relative feature modification costs used in the semi-synthetic and Taiwan experiments. Here, PH = PAY\_HISTORY, BP = BILL\_PAY, SM = SEX\_MAR, and LB = LIMIT\_BAL. Entries with $\infty$ denote immutable features. All finite costs are normalized to have unit mean before constructing the inverse cost matrix.}
\label{tab:cost_profiles}
\end{table*}

The stable model is computed using both RRM and RGD, with RGD employing a step size of $\eta = 5 \times 10^{-4}$. 
Iterations terminate once $\|w^{(t+1)}-w^{(t)}\|_2 < 10^{-8}$ or the maximum number of iterations is reached. We use a maximum of $100$ iterations for RRM and $8000$ iterations for RGD.

We approximate the performatively optimal model using adaptive random search followed by local refinement. Candidate models are sampled uniformly from progressively smaller neighborhoods (radii $1.5$, $0.5$, and $0.15$) around the current best solution using $5{,}000$, $20{,}000$, and $50{,}000$ samples, respectively. The best candidate is then refined through $3{,}000$ iterations of Gaussian perturbation search (i.e. $w' = w + \xi,~\xi \sim \mathcal{N}(0, 0.05^2 I)$), accepting only objective-improving updates. We refer to this as \emph{Grid-search}.

\begin{figure*}[t]
    \centering

    \begin{subfigure}[t]{0.48\textwidth}
        \centering
        \includegraphics[width=\linewidth]{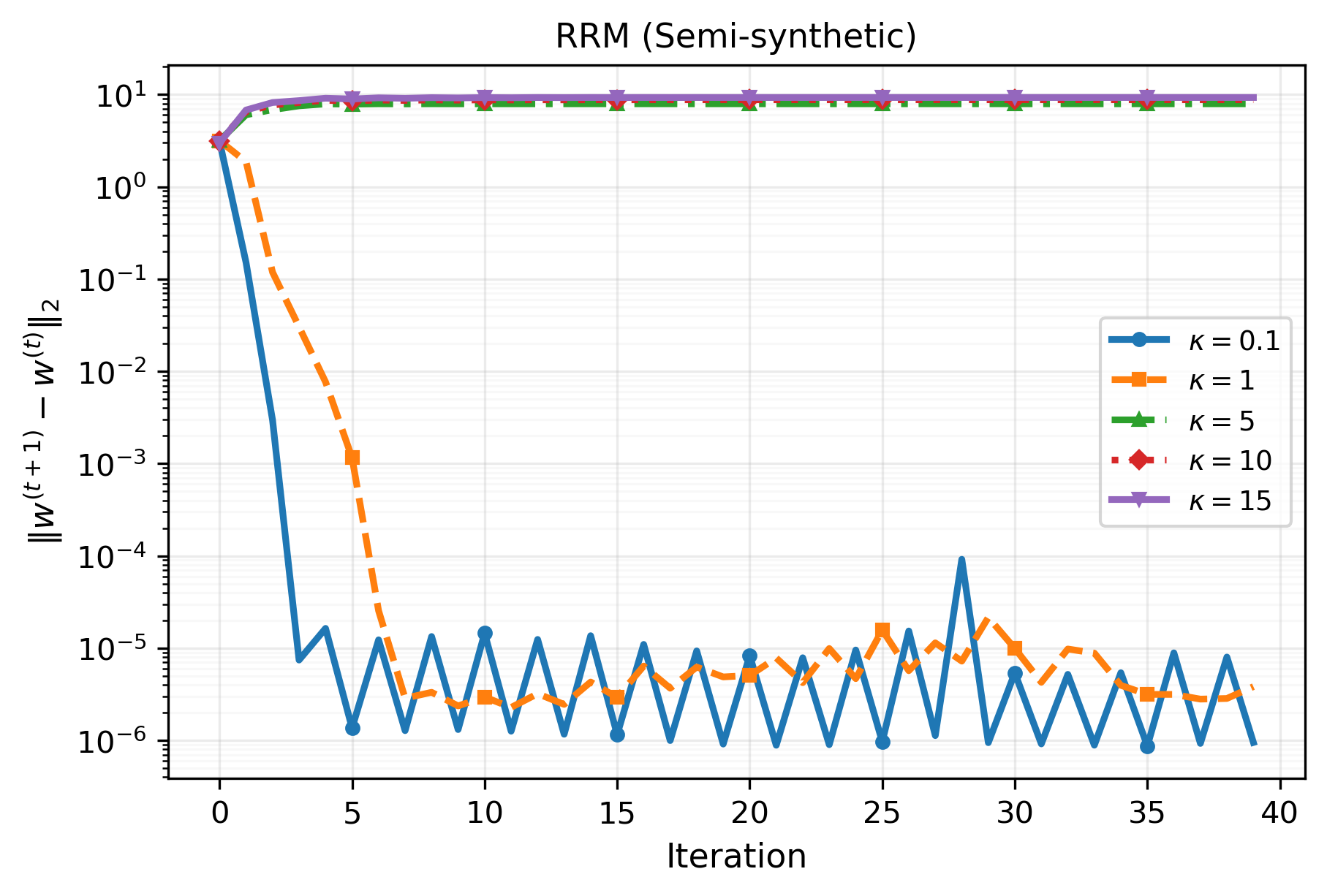}
        \caption{Semi-synthetic (RRM)}
        \label{fig:conv-rrm-synth}
    \end{subfigure}
    \hfill
    \begin{subfigure}[t]{0.48\textwidth}
        \centering
        \includegraphics[width=\linewidth]{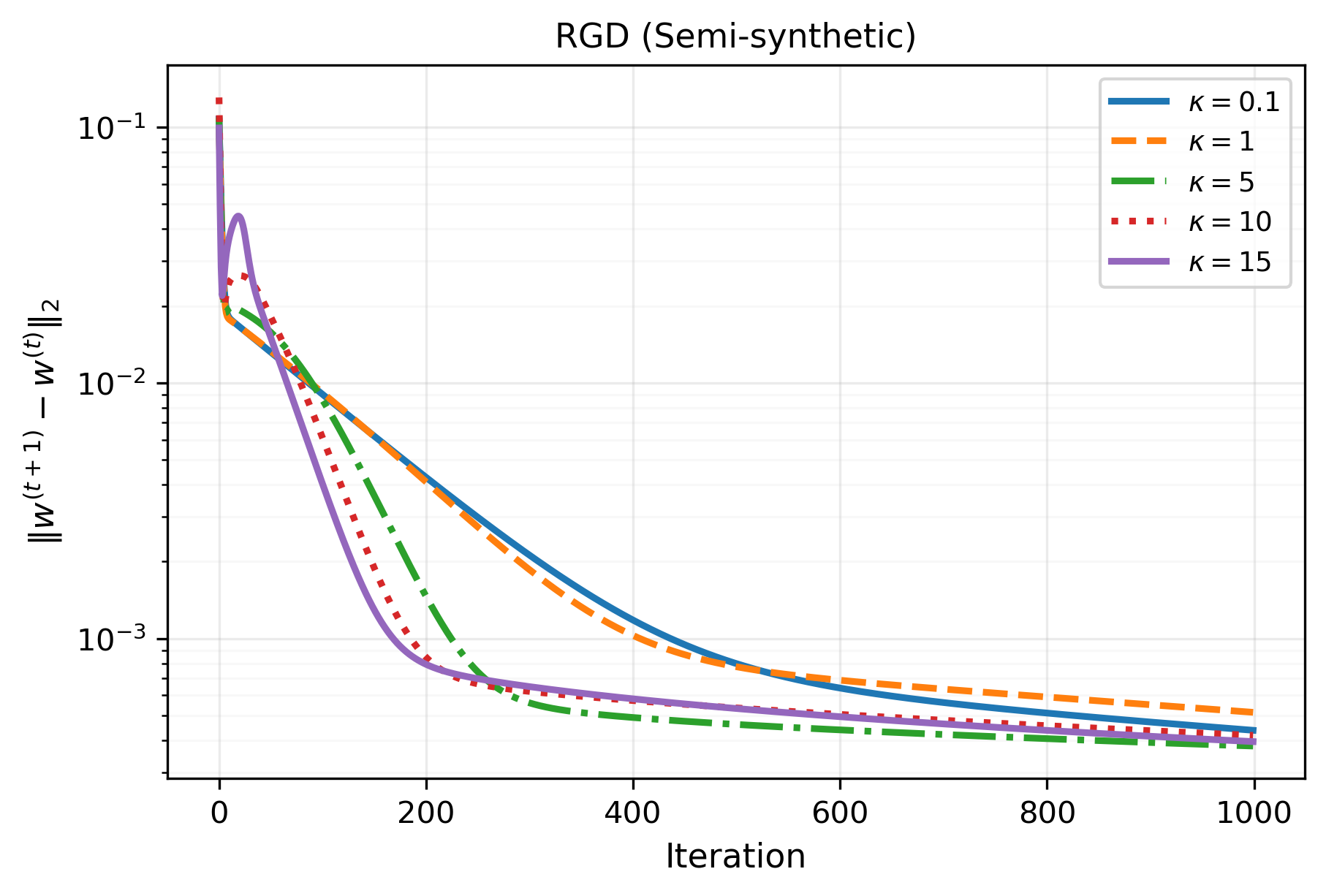}
        \caption{Semi-synthetic (RGD)}
        \label{fig:conv-rgd-synth}
    \end{subfigure}

    \vspace{0.8em}

    \begin{subfigure}[t]{0.48\textwidth}
        \centering
        \includegraphics[width=\linewidth]{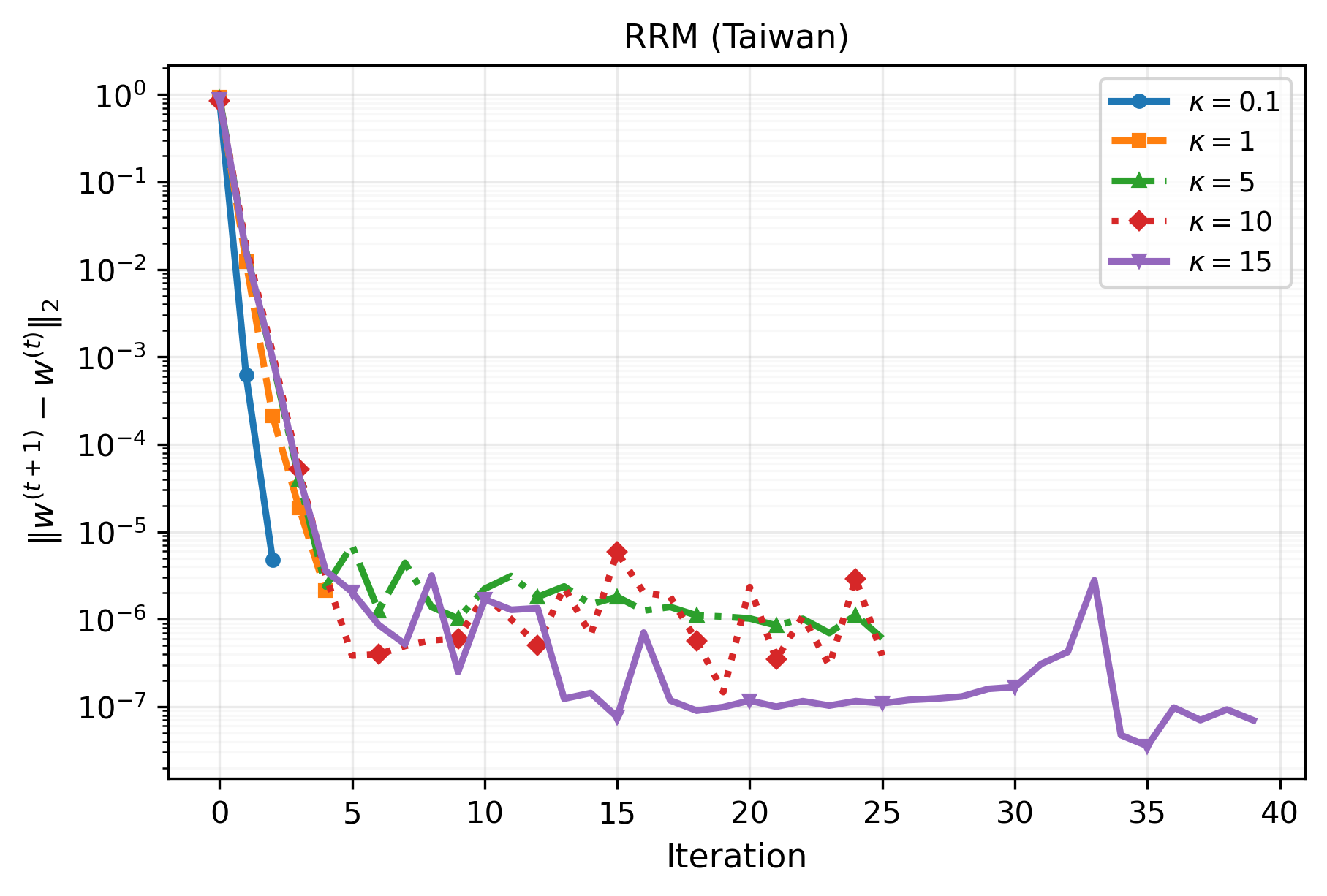}
        \caption{Taiwan (RRM)}
        \label{fig:conv-rrm-taiwan}
    \end{subfigure}
    \hfill
    \begin{subfigure}[t]{0.48\textwidth}
        \centering
        \includegraphics[width=\linewidth]{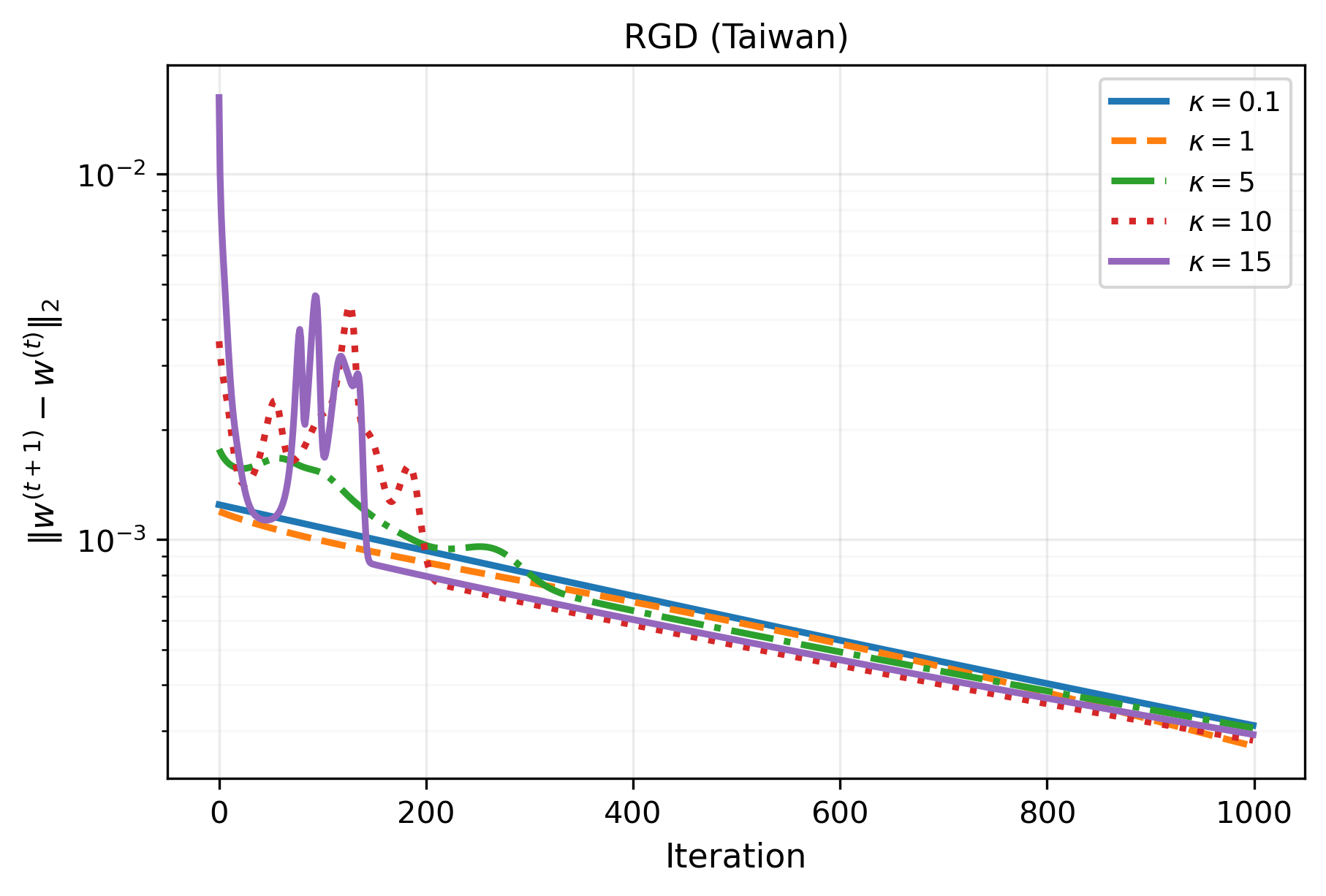}
        \caption{Taiwan (RGD)}
        \label{fig:conv-rgd-taiwan}
    \end{subfigure}

    \caption{
    Convergence of RRM and RGD for varying values of $\kappa$, corresponding to different colored lines. Each curve reports the mean step norm over randomized cost matrix realizations centered around the default cost profile. Top row: semi-synthetic data. Bottom row: Taiwan data.
    }
    \label{fig:convergence}
\end{figure*}

\subsection{Experimental Results}
\label{sec:results}

\begin{figure*}[t]
    \centering

    \begin{subfigure}[t]{0.48\textwidth}
        \centering
        \includegraphics[width=\linewidth]{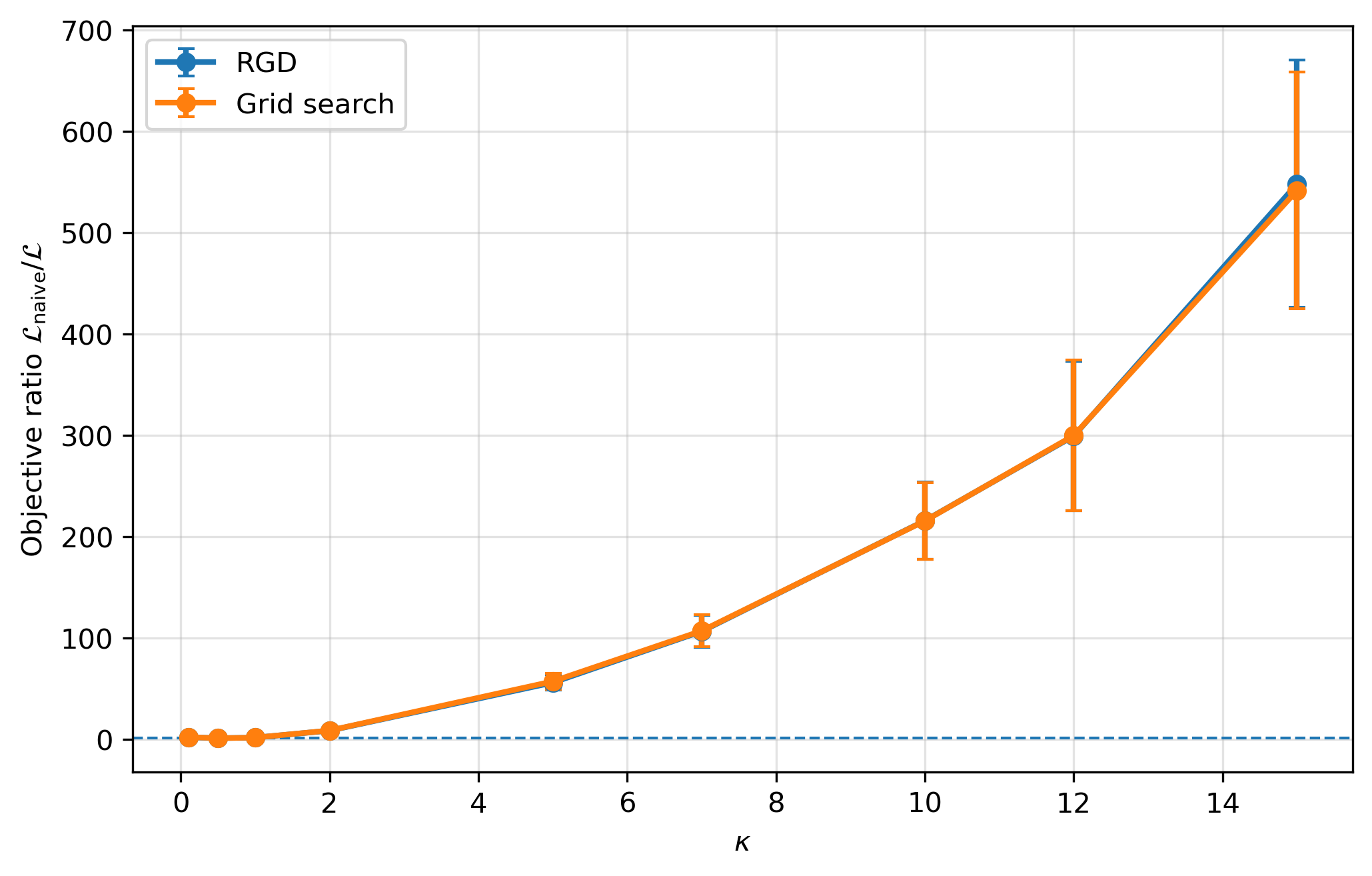}
        \caption{Semi-synthetic dataset.}
        \label{fig:obj-ratio-synth}
    \end{subfigure}
    \hfill
    \begin{subfigure}[t]{0.48\textwidth}
        \centering
        \includegraphics[width=\linewidth]{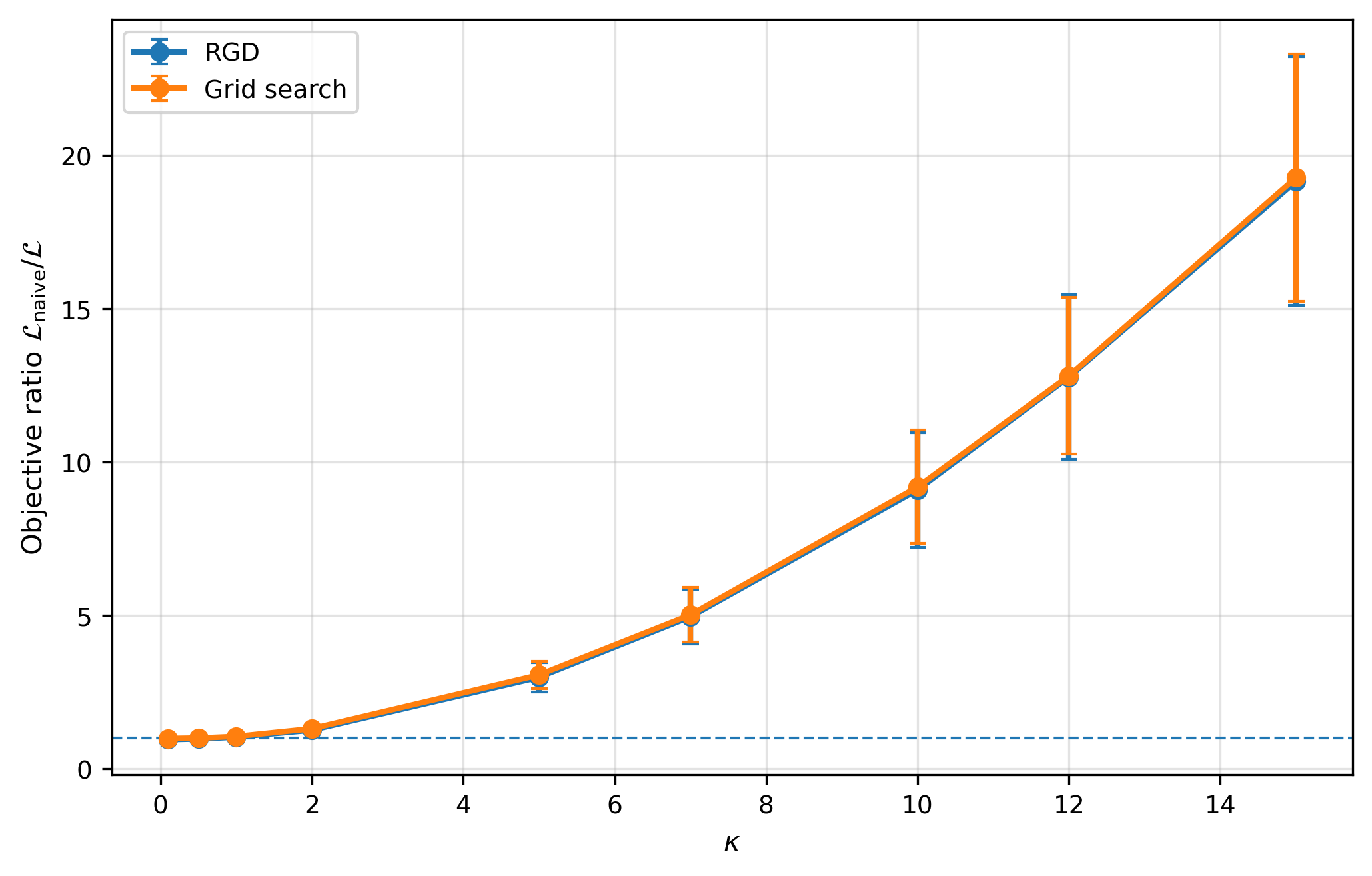}
        \caption{Taiwan dataset.}
        \label{fig:obj-ratio-taiwan}
    \end{subfigure}

    \caption{
    Ratio of the naive ERM performative objective to the corresponding performative objective as a function of the performativity parameter $\kappa$. Results are shown for the stable (RGD) and near-optimal models on the semi-synthetic and Taiwan datasets. Values greater than one indicate improvement over the naive ERM model. Error bars denote one standard deviation over randomized cost matrix realizations.
    }
    \label{fig:rel-improvement}
\end{figure*}

\paragraph{Convergence of RRM and RGD.}
Figure~\ref{fig:convergence} reports the successive iterate difference $\|w^{(t+1)}-w^{(t)}\|_2$ for RRM and RGD across different values of the performativity parameter $\kappa$. On the semi-synthetic dataset, RRM converges rapidly for smaller values of $\kappa$, while for $\kappa\in\{5,10,15\}$ the iterates fail to converge to a fixed point, consistent with the theoretical requirement that the performative sensitivity remain sufficiently small for the RRM contraction guarantee to hold. In contrast, RRM converges within a few iterations for all values of $\kappa$ on the Taiwan dataset.

RGD exhibits a more gradual but consistent decrease in the update norm on both datasets. Although the stopping tolerance is not reached within $1000$ iterations in every setting, the monotonic decay indicates stable optimization progress. The small oscillations visible near the numerical tolerance in the RRM plots are numerical artifacts arising from finite optimization precision.

Using empirically estimated values of the strong-convexity and joint-smoothness constants (see Appendix~\ref{app:bounds}), the sufficient convergence conditions in Corollary~\ref{pro:rrm-stable} and Theorem~\ref{thm:s-0-dist} certify convergence only for $\kappa\lesssim10^{-8}$ on both datasets. This is considerably more conservative than the empirical convergence behavior in Figure~\ref{fig:convergence}, since the analysis relies on a worst-case global estimate of the joint-smoothness constant.

\paragraph{Performative Improvement. } We next evaluate the benefit of accounting for performative effects by comparing the stable (RGD) and performative-optimal (Grid-search) models against the naive Expected Risk Minimization (ERM) baseline in terms of the performative objective, as shown in Figure~\ref{fig:rel-improvement}. As the performativity parameter $\kappa$ increases, the ratio between the performative objective of the naive ERM solution and that of the corresponding stable (or performative-optimal) model also increases, indicating progressively larger improvements over the naive ERM solution. As strategic adaptation becomes stronger, the distribution shift induced by recourse grows, causing the naive ERM model to become increasingly misaligned with the induced data distribution. Consequently, explicitly accounting for performative feedback becomes increasingly important, yielding substantial improvements over standard ERM.

The model distance results in Appendix~\ref{app:sensitiv} further support this conclusion, showing that the relative distance between the stable and Grid-search models remains small across all values of $\kappa$. Together with Figure~\ref{fig:rel-improvement}, this shows that the stable solution is consistently near-optimal in our test scenarios.

\begin{figure*}[t]
    \centering

    \begin{subfigure}[t]{0.48\textwidth}
        \centering
        \includegraphics[width=\linewidth]{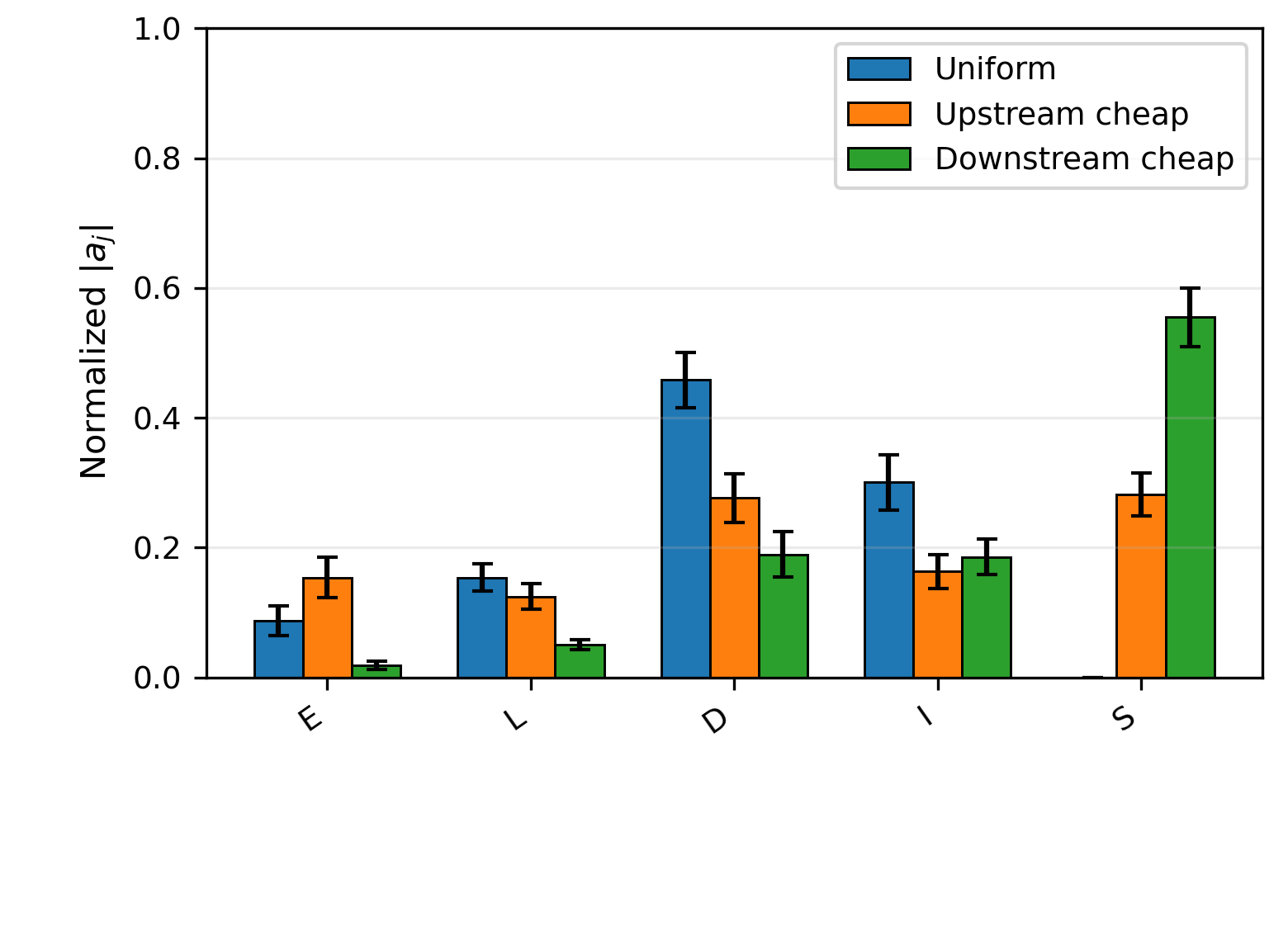}
        \caption{Semi-synthetic: normalized action magnitudes.}
        \label{fig:synth-cost-action}
    \end{subfigure}
    \hfill
    \begin{subfigure}[t]{0.48\textwidth}
        \centering
        \includegraphics[width=\linewidth]{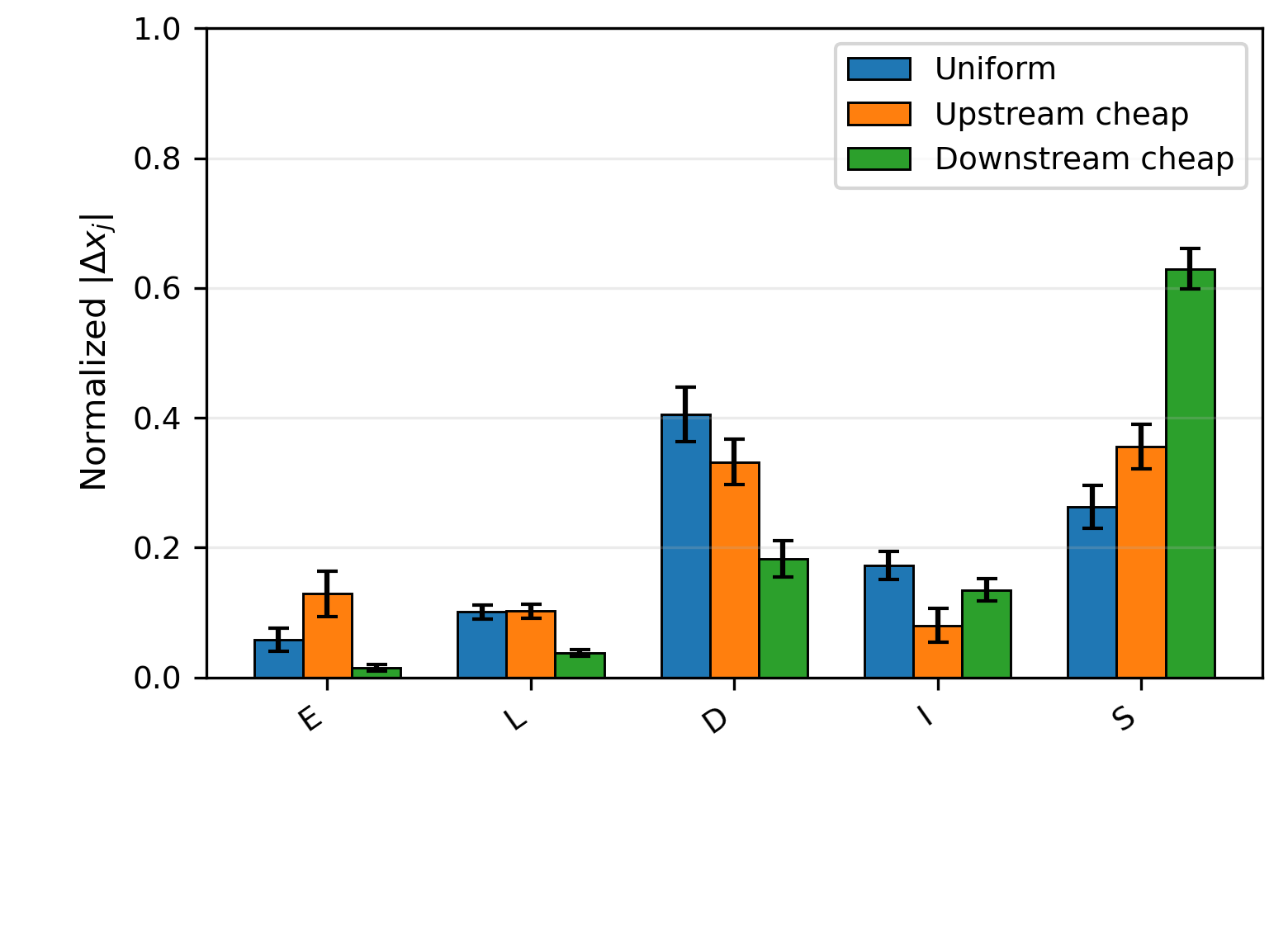}
        \caption{Semi-synthetic: induced feature changes.}
        \label{fig:synth-cost-feature}
    \end{subfigure}

    \vspace{0.8em}

    \begin{subfigure}[t]{0.48\textwidth}
        \centering
        \includegraphics[width=\linewidth]{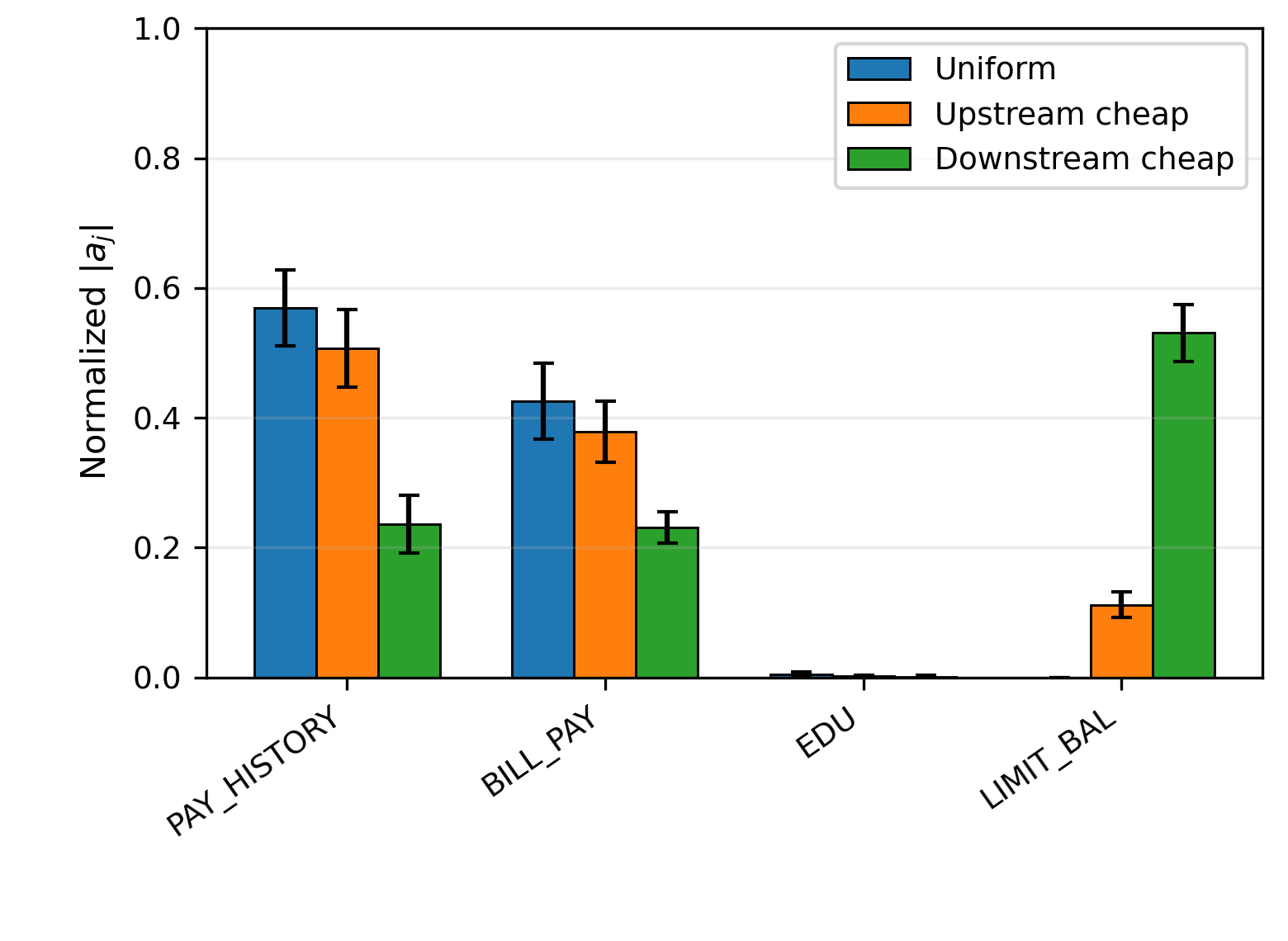}
        \caption{Taiwan: normalized action magnitudes.}
        \label{fig:taiwan-cost-action}
    \end{subfigure}
    \hfill
    \begin{subfigure}[t]{0.48\textwidth}
        \centering
        \includegraphics[width=\linewidth]{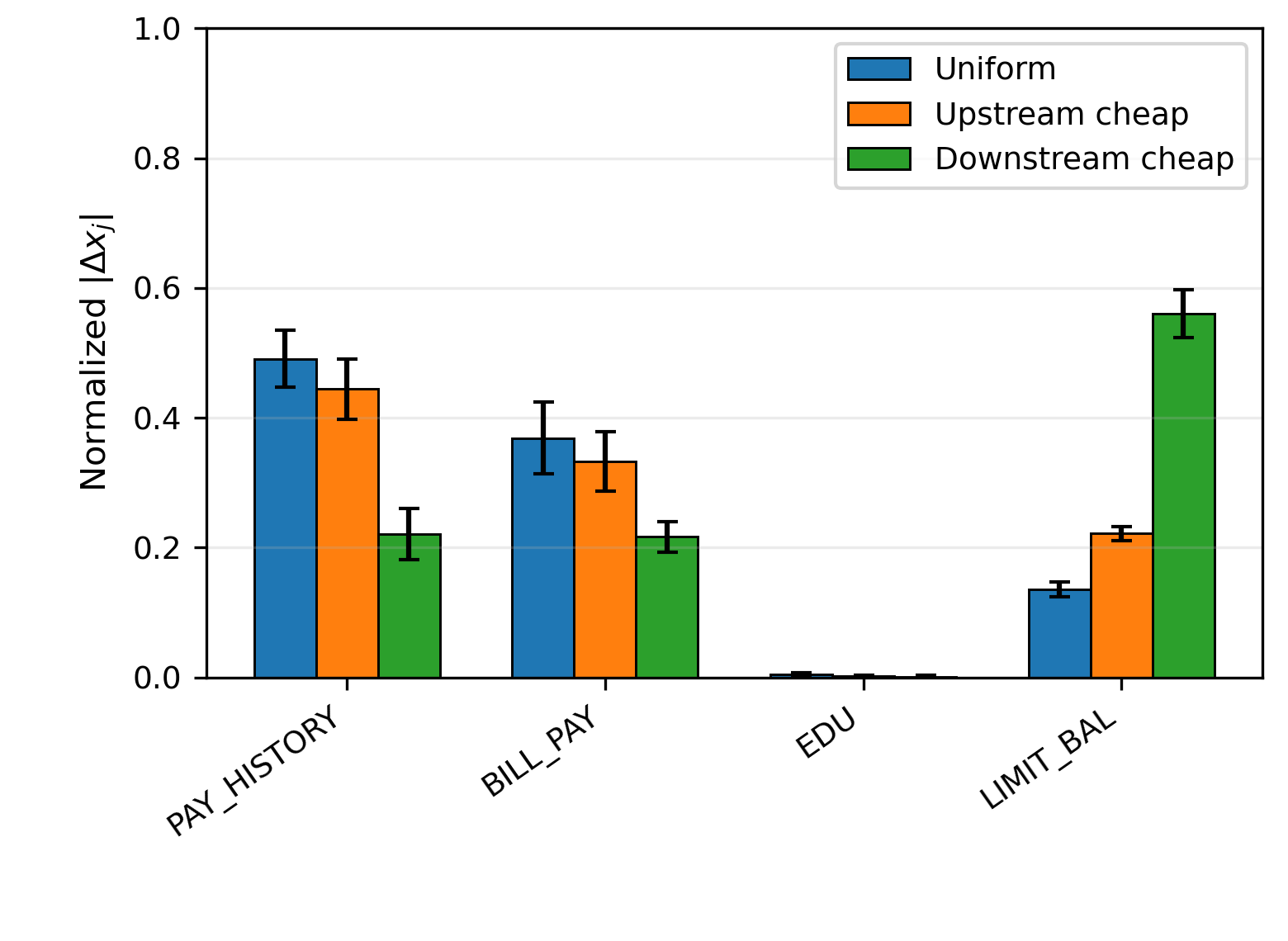}
        \caption{Taiwan: induced feature changes.}
        \label{fig:taiwan-cost-feature}
    \end{subfigure}

    \caption{
    Normalized RGD action magnitudes and induced feature changes on the semi-synthetic and Taiwan datasets for $\kappa=7$ under the \emph{Uniform} (blue), \emph{Upstream-cheap} (orange), and \emph{Downstream-cheap} (green) intervention cost profiles. Bars denote the mean normalized magnitude across randomized cost realizations, with error bars indicating one standard deviation.
    }
    \label{fig:cost-profile-allocation}
\end{figure*}

\paragraph{Causal Propagation.} 
We next investigate how the intervention cost geometry influences the distribution of recourse actions and the resulting feature changes while keeping the causal graph fixed. Specifically, we consider the \emph{Uniform}, \emph{Upstream-cheap}, and \emph{Downstream-cheap} cost profiles from Table~\ref{tab:cost_profiles}. The \emph{Uniform} profile assigns equal costs to all actionable features, the \emph{Upstream-cheap} profile favors interventions on upstream variables, and the \emph{Downstream-cheap} profile favors downstream variables. The corresponding action and feature-change distributions are shown in Figure~\ref{fig:cost-profile-allocation}.

As expected, the intervention cost geometry directly governs the allocation of actions. Under the \emph{Upstream-cheap} profile, the stable model concentrates interventions on upstream variables, whereas the \emph{Downstream-cheap} profile shifts interventions toward downstream variables. The \emph{Uniform} profile yields a more balanced allocation. Without causal propagation (i.e., $G_x$ is the identity matrix), the action and feature-change distributions would coincide. 
Yet, the induced feature changes are governed by causal propagation rather than intervention costs alone. Under the \emph{Uniform} profile, the terminal downstream feature (Savings $S$ in the semi-synthetic dataset and \texttt{LIMIT\_BAL} in the Taiwan dataset) is assigned infinite intervention cost and therefore receives no direct intervention. Nevertheless, both features undergo substantial changes due to upstream interventions propagating through the structural causal model. Consequently, the feature-change distribution differs markedly from the action distribution, indicating the distinct roles of the cost matrix in determining where interventions are applied and the causal graph in determining how those interventions affect the feature space.

\section{Discussion and Future Work}
\label{sec:discussion}

In this work, we introduced a causal performative framework for algorithmic recourse that explicitly models how recommended actions propagate through a structural causal model and influence both observed features and the underlying outcome. We showed that these causal responses naturally induce a performative prediction problem, established conditions under which stable and near-optimal solutions can be efficiently computed, and demonstrated the convergence and practical effectiveness of the proposed approach on both semi-synthetic and real-world datasets.

Our framework assumes access to a structural causal model describing how interventions propagate through the feature space. An important direction for future work is to jointly learn or account for uncertainty in the causal graph while computing stable recourse policies, rather than having access to the graph. Another important direction is extending the framework to richer classes of predictive models and recourse objectives.

\section*{Acknowledgments} 
Avasarala and Ziani thank the National Science Foundation for their support through the CAREER-2336236 and IIS-2504990 awards. Part of this work was supported by the Simons Institute for the Theory of Computing and conducted when Ziani was visiting the Institute.

\bibliographystyle{plainnat}
\bibliography{bib}

\appendix
\appendix
\section{Alternative Improvement-Constrained Action Model}
\label{app:icr}

In our main formulation, an agent selects an action by maximizing the
utility in Equation~\ref{eq:inner}, which trades off predicted score
improvement against intervention cost. We now consider an alternative
action model in which the agent minimizes intervention cost subject to
attaining a prescribed improvement in predicted score. This formulation is closely related to several minimum-cost recourse
formulations in the literature, including improvement-focused causal
recourse \citep{KonigFG23}, but is incorporated here within our
performative learning framework.
Unlike a fixed-model recourse baseline, we incorporate this
alternative action model into our performative framework. Thus, for any
deployed model $\wv$, the agent first computes a minimum-cost
improvement-constrained action, and the learner then accounts for the
distribution shift induced by this action. The resulting stable
model is computed using RRM or RGD, while a performative optimum is
approximated using grid search followed by local refinement.

Let the deployed scoring rule be $h_{\wv}(x)=\wv^\top x$. Under our
causal response model, an action $a\in\mathbb{R}^d$ induces
$x'=x+G_x^\top a$ and $y'=y+G_y^\top a$. Consequently, the increase in
predicted score is
$
h_{\wv}(x')-h_{\wv}(x)
=
\wv^\top G_x^\top a.
$

For a desired score improvement $\bar\gamma>0$, the corresponding action
associated with model $\wv$ is defined as
\begin{equation}
\label{eq:icr-action-problem}
a_I(\wv)
\in
\arg\min_{a\in\mathbb{R}^d}
\left\{
a^\top C a:
\wv^\top G_x^\top a\ge\bar\gamma
\right\}.
\end{equation}

\begin{proposition}[Closed-form improvement-constrained action]
\label{prop:icr-closed-form}
Suppose $C\succ0$, $\bar\gamma>0$, and $G_x\wv\neq0$. Then the
optimization problem in Equation~\ref{eq:icr-action-problem} has a
unique solution given by
\[
a_I(\wv)
=
\frac{\bar\gamma}
{\wv^\top G_x^\top C^{-1}G_x\wv}
C^{-1}G_x\wv.
\]
Moreover, the improvement constraint is active at the optimum:
$
\wv^\top G_x^\top a_I(\wv)=\bar\gamma.
$
\end{proposition}

\begin{proof}
Let $q=G_x\wv$. Since $C\succ0$, the objective is strictly convex. The
KKT stationarity condition gives
$
2Ca-\mu q=0,
$
and hence $a=(\mu/2)C^{-1}q$. Because $\bar\gamma>0$, the constraint is
active at the optimum. Imposing $q^\top a=\bar\gamma$ yields
$
\frac{\mu}{2}
=
\frac{\bar\gamma}{q^\top C^{-1}q},
$
which gives the stated expression. Uniqueness follows from strict
convexity.
\end{proof}

The corresponding feature and label shifts are
$
\Delta x_I(\wv)=G_x^\top a_I(\wv),~
\Delta y_I(\wv)=G_y^\top a_I(\wv).
$
Hence, deployment of $\wv$ maps each observation $(x_i,y_i)$ to
$\left(
x_i+\Delta x_I(\wv),
\;
y_i+\Delta y_I(\wv)
\right)$.
We then define the corresponding performative objective analogously to
Equation~\ref{eq:outer-sample}:
\begin{align}
\wv^*
\in
\arg\min_{\wv\in\Wc}
\frac{1}{m}
\sum_{i=1}^m
\ell\Big(
\wv^\top(&\xv_i+\Delta\xv_I(\av_I)), y_i+\Delta y_I(\av_I)
\Big).
\end{align}

Similarly, the corresponding performatively stable model satisfies
\begin{align}
\wv^s
\in
\arg\min_{\wv\in\Wc}
\frac{1}{m}
\sum_{i=1}^m
\ell\Big(
\wv^\top(&\xv_i+\Delta\xv_I(\wv^s)),   y_i+\Delta y_I(\wv^s)
\Big).
\end{align}

We compute this stable solution using the same RRM and RGD procedures as
in Algorithms~\ref{alg:rrm} and~\ref{alg:rgd}, replacing the
utility-maximizing action from Equation~\ref{eq:inner} with
$a_I(\wv)$ from Proposition~\ref{prop:icr-closed-form}.

The corresponding control parameters admit complementary
interpretations. In our main framework, the parameter $\kappa$ scales
the inverse cost matrix and controls the overall actionability of the
environment, with larger values of $\kappa$ leading to stronger recourse
actions and consequently greater performative effects. In contrast,
$\bar\gamma$ specifies the desired improvement that the prescribed
action must achieve. Consequently, $\kappa$ governs \emph{how much
agents can change}, whereas $\bar\gamma$ governs \emph{how much
improvement is required}. These parameters therefore capture
complementary aspects of the recourse problem: actionability versus
desired improvement.

\section{Theory}
\label{app:proof}

\subsection{Structural Equations and Causal Propagation }
\label{app:scm}

We model the causal relationships among the input features using a Structural Causal Model (SCM). Let $\xv=(x_1,\ldots,x_d)^\top$ denote the vector of input features. Each feature $x_i$ is generated according to a structural equation of the form
$
x_i=f_i(\mathrm{Pa}(x_i),u_i),
$
where $\mathrm{Pa}(x_i)$ denotes the set of direct causal parents of $x_i$ in the causal graph and $u_i$ is an exogenous noise variable. Throughout the paper, we consider a linear SCM, which can be written compactly as $\xv=\Av^\top \xv+u$, where $\Av$ is the weighted adjacency matrix of the causal graph. Entry $\Av_{ij}$ represents the \emph{direct} causal influence of feature $i$ on feature $j$.

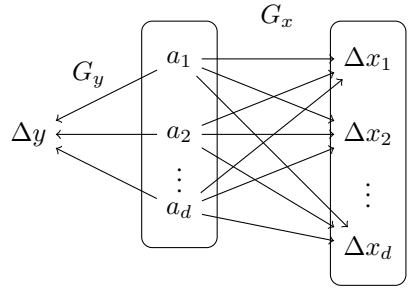
\begin{figure}[ht!]
\begin{center}
\begin{tikzpicture}[scale=1.0]

    \node (y) at (0,1) {$\Delta y$};
    
    \draw[rounded corners] (1.5,2.5) rectangle (2.5,-0.5);
    \node (a1) at (2,2) {$a_1$};
    \node (a2) at (2,1) {$a_2$};
    \node (dotsA) at (2,0.5) {$\vdots$};
    \node (am) at (2,0) {$a_d$};
    
    \draw[rounded corners] (4,2.5) rectangle (5,-1);
    \node (x1) at (4.5,2) {$\Delta x_1$};
    \node (x2) at (4.5,1) {$\Delta x_2$};
    \node (dotsX) at (4.5,0.3) {$\vdots$};
    \node (xn) at (4.5,-0.5) {$\Delta x_d$};
    
    \draw[<-] (y) -- (a1);
    \draw[<-] (y) -- (a2);
    \draw[<-] (y) -- (am);
    \draw[->] (a1) -- (x1);
    \draw[->] (a1) -- (x2);
    \draw[->] (a1) -- (xn);
    \draw[->] (a2) -- (x1);
    \draw[->] (a2) -- (x2);
    \draw[->] (a2) -- (xn);
    \draw[->] (am) -- (x1);
    \draw[->] (am) -- (x2);
    \draw[->] (am) -- (xn);
    
    \node[above] at (0.8,1.5) {\textit{$G_y$}};
    \node[above] at (3.3,2.3) {\textit{$G_x$}};
\end{tikzpicture}
\end{center}
\caption{The actions $\av$ and how they affect features and label changes using a linear relationship, i.e., $\Delta\xv(\av)=G_x^T\av$ and $\Delta y(\av)= G_y^T\av$. }
\label{fig:diagram of system}
\end{figure}

We model recourse actions as \emph{additive interventions} on the structural equations. Given an intervention vector $\av=(a_1,\ldots,a_d)^\top$, the post-intervention features satisfy
$
\xv'=\Av^\top \xv'+u+\av.
$
Thus, the intervention $a_i$ acts as an additive shift to the structural equation generating feature $i$.

Unlike independent feature perturbations, an intervention on one feature propagates through the causal graph and may induce changes in downstream features. Besides direct effects, interventions also produce indirect effects through multi-hop causal paths. For example, if feature $i$ influences an intermediate feature $k$, which in turn influences feature $j$, then the indirect contribution from $i$ to $j$ is given by the product $\Av_{ik}\Av_{kj}$. More generally, higher powers of the adjacency matrix capture the effects of increasingly long causal paths.

Defining the feature change as $\Delta \xv=\xv'-\xv$, subtracting the pre- and post-intervention structural equations gives
\begin{align}
\Delta \xv
&=\Av^\top\Delta \xv+\av
=(I-\Av^\top)^{-1}\av.
\end{align}
We therefore define the feature contribution matrix as
$
G_x^\top:=(I-\Av^\top)^{-1},
$
so that
$
\Delta \xv=G_x^\top \av.
$
Thus, the matrix $G_x^\top$ captures both the direct and indirect effects of interventions, mapping an intervention vector to the resulting changes in the features. The following result provides a closed-form characterization of $G_x$ for directed acyclic graphs.

\begin{remark}
\label{rem:Gx}
For a directed acyclic graph, $G_x=\sum_{k=0}^{d-1}\Av^k$.

\end{remark}

\begin{proof}[Proof of Remark~\ref{rem:Gx}]
We first show that, for every integer $k\ge0$, the $(i,j)$-th entry of $\Av^k$
equals the total contribution of all directed paths of length exactly $k$
from feature $i$ to feature $j$. The proof is by induction on $k$.

For the base case, $k=0$, we have $\Av^0=I$. Hence,
\[
(\Av^0)_{ij}=
\begin{cases}
1,& i=j,\\
0,& i\neq j,
\end{cases}
\]
which corresponds to the unique path of length zero from a node to itself.

Suppose the claim holds for some $k\ge0$. By matrix multiplication,
\[
(\Av^{k+1})_{ij}
=
\sum_{\ell=1}^{d}
(\Av^k)_{i\ell}\Av_{\ell j}.
\]
By the induction hypothesis, $(\Av^k)_{i\ell}$ aggregates the contribution of
all directed paths of length $k$ from feature $i$ to feature $\ell$, while
$\Av_{\ell j}$ represents the direct influence from feature $\ell$ to feature $j$.
Therefore, each product $(\Av^k)_{i\ell}\Av_{\ell j}$ corresponds to extending every
length-$k$ path ending at $\ell$ by one additional edge to $j$. Summing over
all intermediate features $\ell$ accounts for every directed path of length
exactly $k+1$ from $i$ to $j$, establishing the induction.

Since the underlying causal graph is acyclic, every directed path contains at
most $d-1$ edges. Consequently, $\Av^k=0$ for all
$k\ge d$, and summing over all feasible path lengths gives
\[
G_x
=
\sum_{k=0}^{d-1}
\Av^k.
\]

Finally, computing $G_x$ requires evaluating the powers
$\Av,\Av^2,\ldots,\Av^{d-1}$.
Each matrix multiplication requires $O(d^3)$ operations, and
there are $O(d)$ such products. Hence, the contribution matrix can
be computed in polynomial time.
\end{proof}

\subsection{Sensitivity Characterization }
\label{app:sens}

\begin{definition}
\label{def:distribution-sensitivity}
A distribution map $D_{(\cdot)}$ is $\epsilon$-sensitive if
\[
\emd(D_{\wv},D_{\wv'})
\leq
\epsilon\|\wv-\wv'\|_2,
\qquad
\text{for all } \wv,\wv'\in\Wc,
\]
where $\emd$ denotes the earth mover's distance, equivalently the
Wasserstein-$1$ distance with $\ell_2$ transportation cost.
\end{definition}

The following lemma provides the main connection between our causal
recourse model and the performative prediction framework. In
particular, it explicitly characterizes the sensitivity parameter in
terms of the feature contribution matrices and the intervention cost
matrix.

\begin{lemma}
\label{lem:distribution-sensitive}
Under Assumption~\ref{assumpt:cost}, the distribution map
$D_{(\cdot)}$ is $\epsilon$-sensitive with
$
\epsilon
=
\frac{\kappa}{2}\Lambda(C,G_x,G_y),
$
where
$\Lambda(C,G_x,G_y)
:=
\|G_x\|_2^2\|C^{-1}\|_2
+
\|G_x\|_2\|C^{-1}\|_2\|G_y\|_2.$

\end{lemma}
\begin{proof}
 Define $A=G_x^T\inv{C}G_x/(2\lambda)$ and $B=G_y^T\inv{C}G_x/(2\lambda)$. Any $(\xv,y)\sim D$ will get mapped to $(\xv+A\wv,y+B\wv)$ and  $(\xv+A\wv',y+B\wv')$ when using $\wv$ and $\wv'$, respectively. Hence, the earth mover's distance is bounded by
 \begin{align*}
    \emd\left(D_{\wv}, D_{\wv'}\right) &\leq \|\left(A\left(\wv-\wv\right)',\left(\wv-\wv'\right)B\right)\|_2 \\ 
    &\leq \|(A\left(\wv-\wv\right)'\|_2+\|\left(\wv-\wv'\right)B)\|_2 \\ 
    &\leq \|A\|_2\|\wv-\wv'\|_2+\|\wv-\wv'\|_2\|B\|_2 \\ 
    &= \left(\|A\|_2+\|B\|_2\right) \|\wv-\wv'\|_2,
\end{align*}
where the first inequality is by the definition of the earth mover's distance, the second is by the triangle inequality, and the last is by the operator norm inequality. The proof is completed by noting that $\|A\|_2\leq\|G_x\|_2^2\|\inv{C}\|_2/(2\lambda)$ and $\|B\|_2\leq\|G_x\|_2\|\inv{C}\|_2\|G_y\|_2/(2\lambda)$
\end{proof}

Lemma~\ref{lem:distribution-sensitive} shows that 
stronger causal amplification,
corresponding to larger $\|G_x\|_2$ or $\|G_y\|_2$, and lower
intervention costs, corresponding to larger $\kappa$ or
$\|C^{-1}\|_2$, increase the sensitivity of the induced distribution.

\subsection{Missing Proofs}
\label{app:proofs}
\begin{proof}[Proof of Lemma~\ref{pro:recourse-closed-form}]
The agent solves
\[
\max_{\av} \; \wv^\top(\xv+G_x^\top \av)-\frac{1}{\kappa} \av^\top C \av .
\]
Dropping the term independent of $\av$, this is equivalent to maximizing
\[
\wv^\top G_x^\top \av-\frac{1}{\kappa} \av^\top C \av .
\]
Since $C$ is positive definite, the objective is strictly concave in $\av$. The first-order condition gives
\[
G_x \wv - 2\frac{1}{\kappa} C \av = 0,
\]
and hence
\[
\av^*=\frac{\kappa}{2}C^{-1}G_x \wv .
\]
\end{proof}
\begin{proof}[Proof of Lemma~\ref{pro:non-convex}]
Consider the special case where $\ell$ is the squared loss and consider a sample $S=\{(x,y)\}$ which contains a single labeled example in one dimension with one available action, i.e., $d=1$ and $n=1$. For this case, the optimization problem in Equation~\ref{eq:outer-sample-simplified} can be simplified as 
\begin{align*}
\min_{w\in\mathbb{R}} \quad f(w) \text{ where } f(w)&=\left(w(x+aw)-(y+bw)\right)^2=\left(aw^2+(x-b)w-y\right)^2,
\end{align*}
where $a=\kappa G_x^2/(2C)$ and 
$b=\kappa G_yG_x/(2C)$ are scalars as $G_x$, $G_y$ and $C (\inv{C}$) all become scalars when $d=n=1$.  The function $f(w)$ is non-convex when $a=1, x=0, b=1, y=1$ (can be verified by plotting $(x^2-x-1)^2$). This combination can be achieved by setting $C=1$, $G_x=1$, $G_y=1$, and $\kappa = 2$.
\end{proof}

\begin{proof}[Proof of Theorem~\ref{pro:rrm-stable}]
 The proof is a direct corollary of Theorem 3.10 in~\cite{PerdomoZMH20} and Lemma~\ref{lem:distribution-sensitive}.
\end{proof}

\begin{proof}[Proof of Theorem~\ref{thm:s-0-dist}]
Under Assumptions~\ref{assumpt:cost},~\ref{assumpt:ell-strong-convex}, ~\ref{assumpt:ell-smooth}, ~and the stated additional assumption, from Theorem 4.3 of \citep{PerdomoZMH20} performatively stable point and performative optimum satisfy
$
\|\wv^s-\wv^*\|_2
\le
\frac{2L_z\epsilon}{\gamma}.
$
Further, by Theorem~\ref{thm:s-0-dist}, under the respective conditions for RRM and RGD, we have
$
\|\wv^T-\wv^s\|_2 \le \delta
$
with probability at least $1-p$ after the stated number of iterations, where $\wv^s$ denotes the performatively stable point. Therefore, by the triangle inequality,
\begin{align}
\|\wv^T-\wv^*\|_2
&\le
\|\wv^T-\wv^s\|_2
+
\|\wv^s-\wv^*\|_2 \nonumber\le
\delta+\frac{2L_z\epsilon}{\gamma}. \nonumber
\end{align}
\end{proof}

\begin{figure*}[t]
    \centering

    \begin{subfigure}[t]{0.48\textwidth}
        \centering
        \includegraphics[width=\linewidth]{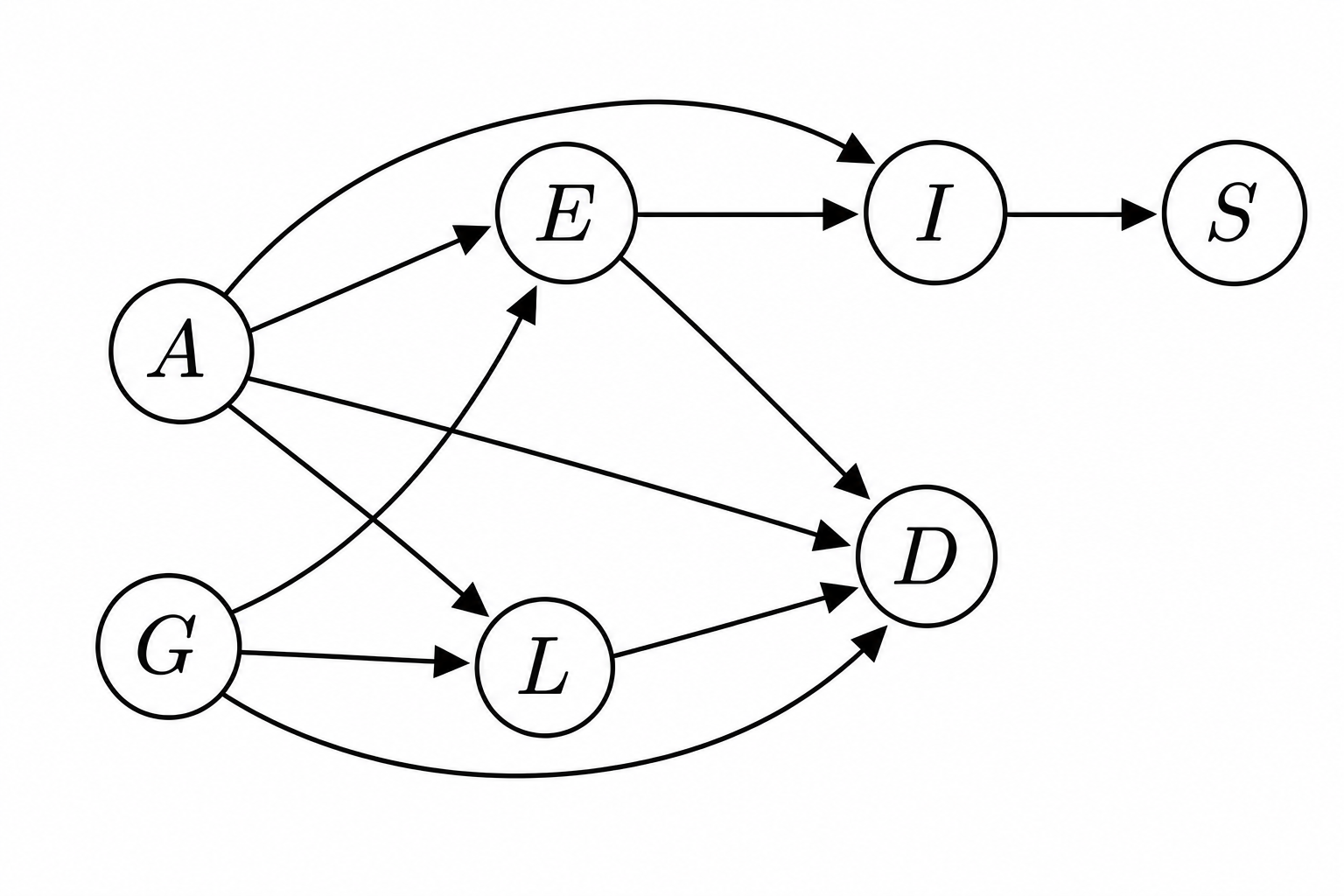}
        \caption{}
        \label{fig:graph-syn}
    \end{subfigure}
    \hfill
    \begin{subfigure}[t]{0.48\textwidth}
        \centering
        \includegraphics[width=\linewidth]{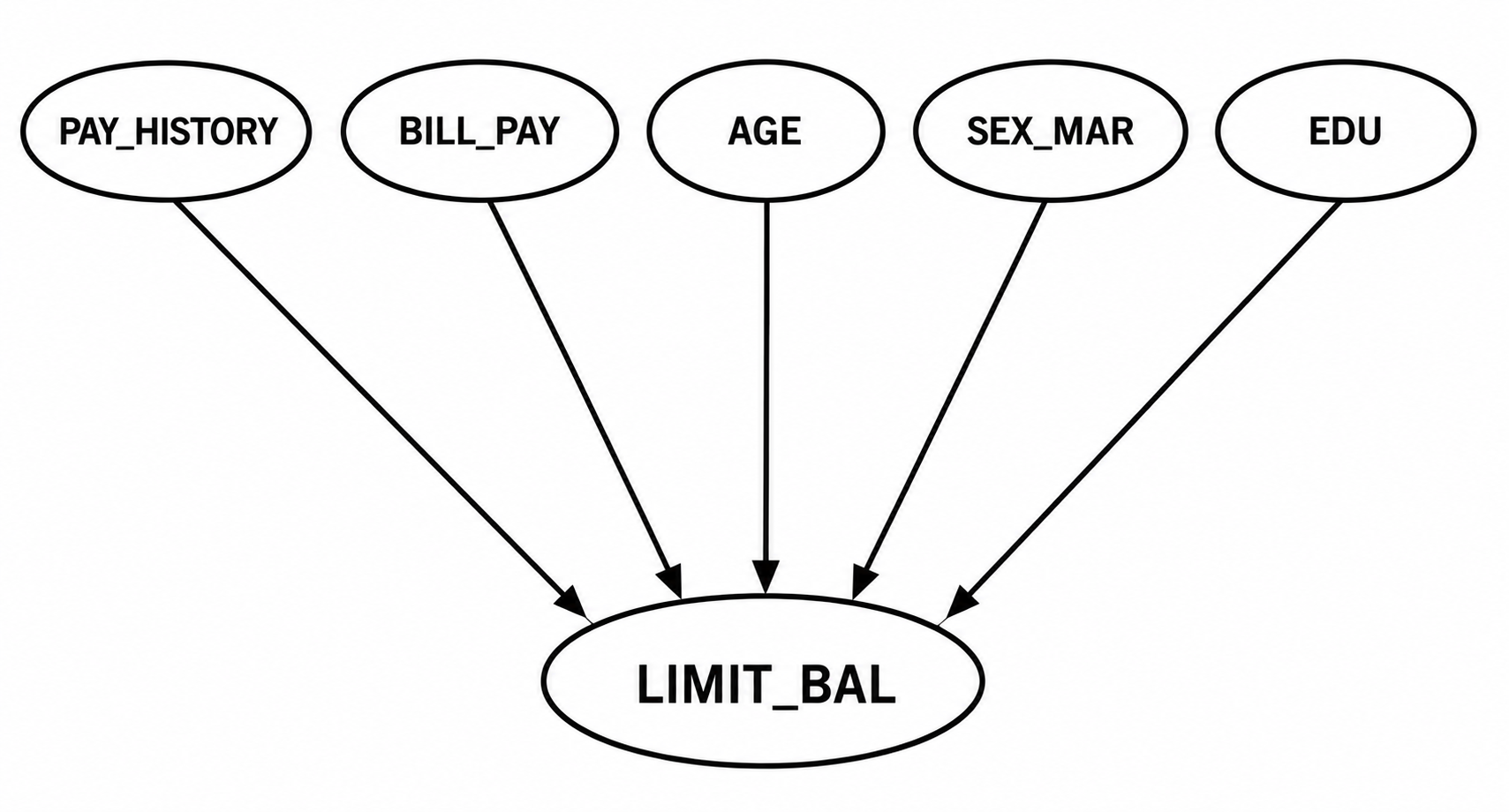}
        \caption{}
        \label{fig:graph-tai}
    \end{subfigure}

    \caption{
    Causal graph for a) Semi-synthetic dataset, and b) Taiwan dataset
    }
    \label{}
\end{figure*}

\subsection{Computation of Theoretical Bounds}
\label{app:bounds}
We next specialize the strong-convexity and joint-smoothness constants to the regularized squared loss used in our experiments. For a sample $(x,y)$, let
\[
\ell(\wv;x,y)
=
(\wv^\top x-y)^2+\rho\|\wv\|_2^2,
\]
and define the empirical objective
\[
L(\wv)
=
\frac{1}{m}\sum_{i=1}^{m}\ell(\wv;x_i,y_i)
=
\frac{1}{m}\|X\wv-y\|_2^2+\rho\|\wv\|_2^2,
\]
where $X\in\mathbb{R}^{m\times d}$ and $y\in\mathbb{R}^m$.

\paragraph{Strong Convexity.}
Since
\[
\nabla_{\wv}^{2}L(\wv)
=
\frac{2}{m}X^\top X+2\rho I
\succeq 2\rho I,
\]
the objective is $\gamma$-strongly convex with
$\gamma\geq 2\rho$. We therefore use the distribution-independent
choice $\gamma=2\rho$.

\paragraph{Joint Smoothness.}
Assume that $\|\wv\|_2\leq B_w$, $\|x\|_2\leq B_x$, and
$|y|\leq B_y$. The gradient of the pointwise loss is
\[
\nabla_{\wv}\ell(\wv;x,y)
=
2(\wv^\top x-y)x+2\rho\wv.
\]
For fixed $(x,y)$, its variation with respect to $\wv$ satisfies
\[
\|\nabla_{\wv}\ell(\wv;x,y)
-
\nabla_{\wv}\ell(\wv';x,y)\|_2
\leq
(2B_x^2+2\rho)\|\wv-\wv'\|_2.
\]
For fixed $\wv$, its variation with respect to the data satisfies
\begin{align}
\left\|
\nabla_{\wv}\ell(\wv;z)
-
\nabla_{\wv}\ell(\wv;z')
\right\|_2
&\le
\left(4B_xB_w+2B_y+2B_x\right)\|z-z'\|_2, \nonumber
\end{align}
where $z=(x,y), z'=(x',y'),$
Thus, a valid joint-smoothness constant is
\[
\beta
=
\max\left\{
2B_x^2+2\rho,\;
4B_xB_w+2B_y+2B_x
\right\}.
\]
For each value of $\kappa$, we record the optimization trajectories of both RRM and RGD. At every iteration, we compute the induced dataset and evaluate the largest feature norm, score magnitude, and model norm. The maximum values observed over the trajectories are used as $B_x$, $B_y$, and $B_w$, respectively.
\section{Datasets}
\label{app:data}
\subsection{Semi-synthetic Dataset}
\label{app:synth}
We use the same structural causal model as \citep{karimi2020algorithmic}, while adapting the labels to a regression setting. The feature vector is given by
\[
x = (G,A,E,L,D,I,S),
\]
where $G$ denotes gender, $A$ age, $E$ education level, $L$ loan amount, $D$ loan duration, $I$ income, and $S$ savings. The variables are generated according to the structural equations
\begin{align*}
G &= 0.5 + \varepsilon_G \\
A &= 0 + \varepsilon_A\\
E &= -0.062 + 0.123G + 0.0062A + \varepsilon_E \\
L &= 1.25 + 0.1A + G + \varepsilon_L \\
D &= 1.25 + 0.1A + 2G + L + \varepsilon_D \\
I &= 0.469 + 0.1A + 1.938G + 0.5E + \varepsilon_I \\
S &= -3.2965 + 1.5I + \varepsilon_S ,
\end{align*}
where $\varepsilon_G \sim \mathrm{Bernoulli}(0.5)$,
$\varepsilon_A \sim \mathrm{Gamma}(10,3.5)$,
$\varepsilon_E \sim \mathcal{N}(0,0.5)$,
$\varepsilon_L \sim \mathcal{N}(0,2)$,
$\varepsilon_D \sim \mathcal{N}(0,3)$,
$\varepsilon_I \sim \mathcal{N}(0,2)$, and
$\varepsilon_S \sim \mathcal{N}(0,5)$ are mutually independent noise variables. The continuous outcome score is generated as
\[
y
=
-0.3\left(-L-D+I+S+IS\right),
\]
which defines the loan-approval score used in the MSE objective.

\subsection{Taiwan Dataset}
\label{app:taiwan}
We construct six features from the Taiwan credit dataset. The feature \texttt{PAY\_HISTORY} is defined as the average of the historical repayment status variables \texttt{PAY\_1},\ldots,\texttt{PAY\_6}. To summarize repayment behavior, we define \texttt{BILL\_PAY}. Let $\overline{\mathrm{BILL}}$ and $\overline{\mathrm{PAY}}$ denote the averages of \texttt{BILL\_AMT1},\ldots,\texttt{BILL\_AMT6} and \texttt{PAY\_AMT1},\ldots,\texttt{PAY\_AMT6}, respectively, and let \texttt{LIMIT\_BAL} denote the assigned credit limit. Following the preprocessing of~\cite{pitso2026bayesian}, we define
\[
\texttt{BILL\_PAY}
=
\frac{1}{2}\frac{\overline{\mathrm{BILL}}}{\texttt{LIMIT\_BAL}}
+
\frac{1}{2}\frac{\overline{\mathrm{PAY}}}{\overline{\mathrm{BILL}}}.
\]
The remaining features are \texttt{AGE}, \texttt{LIMIT\_BAL}, \texttt{EDU}, obtained directly from the education category, and \texttt{SEX\_MAR}, formed by combining the sex and marital-status variables. Samples containing invalid or missing values are discarded.

\paragraph{Continuous label generation.}
The original Taiwan credit dataset provides a binary default label. Since our objective is to model how repayment likelihood changes under strategic feature modifications, we instead construct a continuous target. Let $X\in\mathbb{R}^{n\times d}$ denote the feature matrix and $Y\in\{0,1\}^n$ the default labels. After standardizing the features, we fit an $\ell_2$-regularized logistic regression model to predict $Y$ and use the predicted probabilities as continuous semisynthetic labels,
$\tilde y_i=\Pr(Y_i=1\mid x_i)$.

\begin{figure*}[t]
    \centering

    \includegraphics[width=\textwidth]{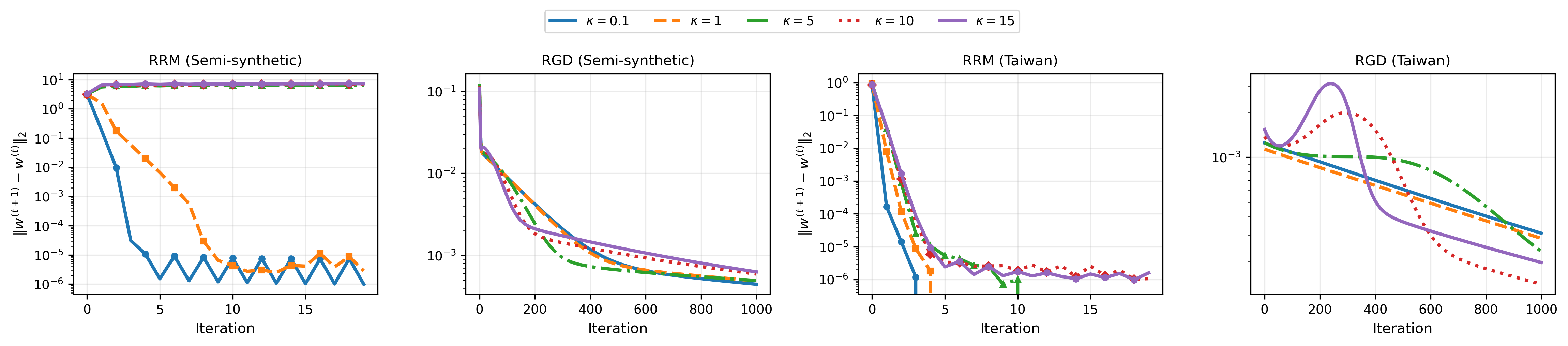}

    \vspace{0.3em}

    \includegraphics[width=\textwidth]{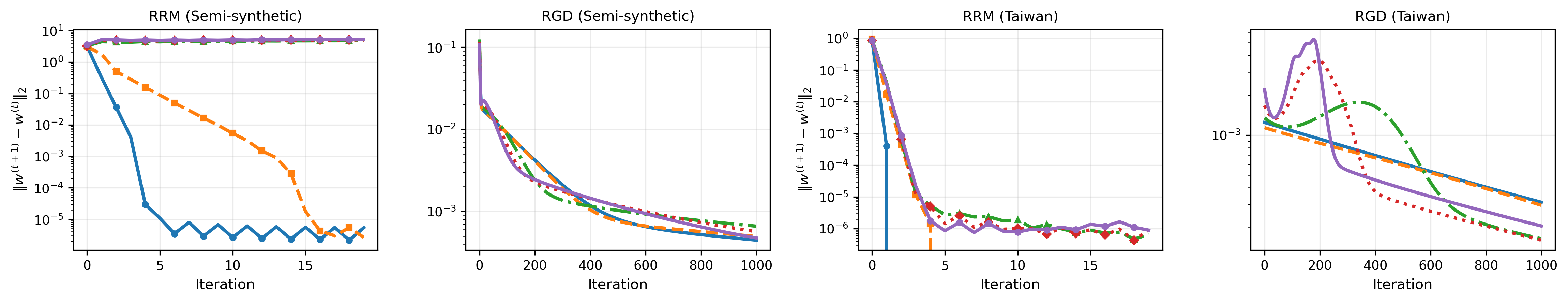}

    \vspace{0.3em}

    \includegraphics[width=\textwidth]{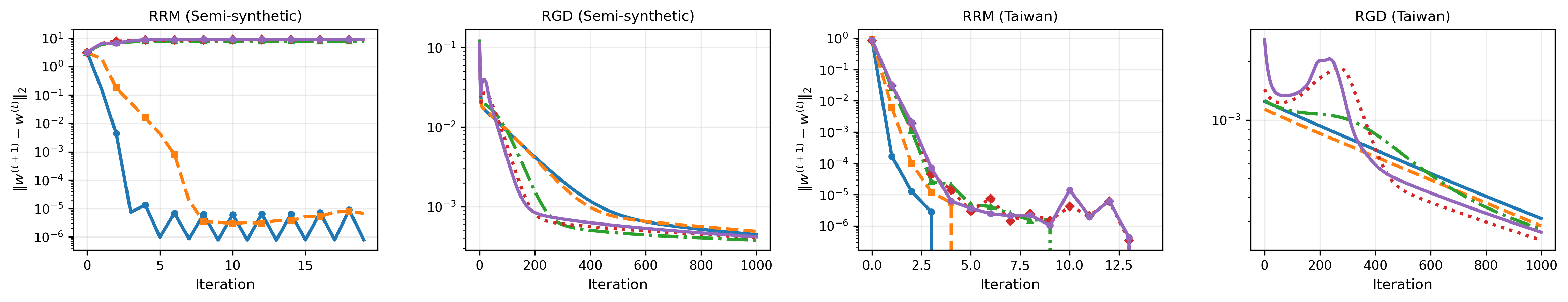}

\caption{
Convergence behavior of RRM and RGD under randomized cost matrices for the semi-synthetic and Taiwan datasets. The columns correspond to RRM (semi-synthetic), RGD (semi-synthetic), RRM (Taiwan), and RGD (Taiwan), respectively. Solid curves denote the mean over randomized cost realizations. The top, middle, and bottom rows correspond to the uniform, upstream-cheaper, and downstream-cheaper cost profiles, respectively (Table~\ref{tab:cost_profiles}).
}
    \label{fig:all-convergence}
\end{figure*}

\begin{figure*}[t!]
    \centering

    \begin{subfigure}[t]{0.32\textwidth}
        \centering
        \includegraphics[width=\linewidth]
        {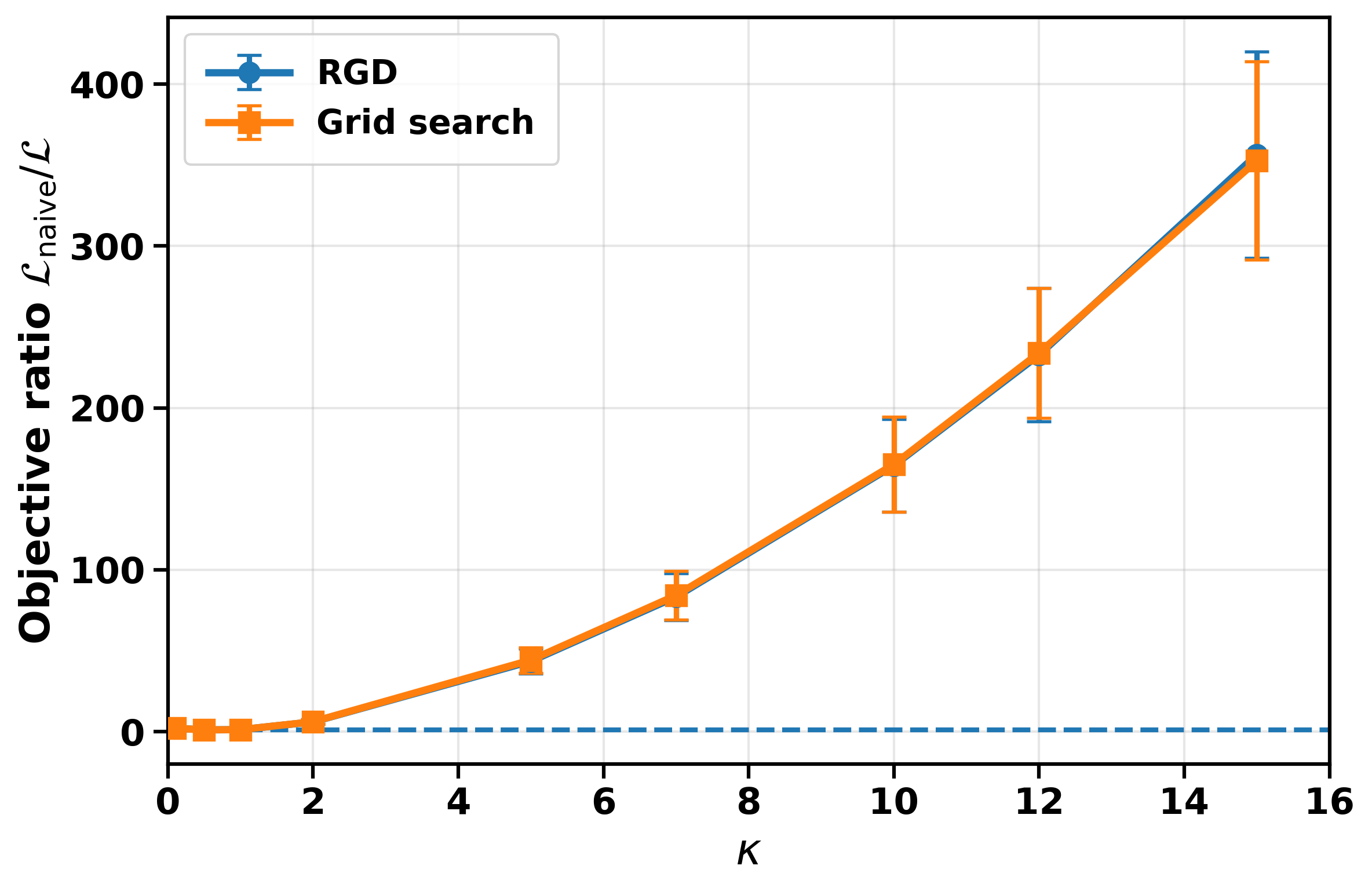}
    \end{subfigure}
    \hfill
    \begin{subfigure}[t]{0.32\textwidth}
        \centering
        \includegraphics[width=\linewidth]
        {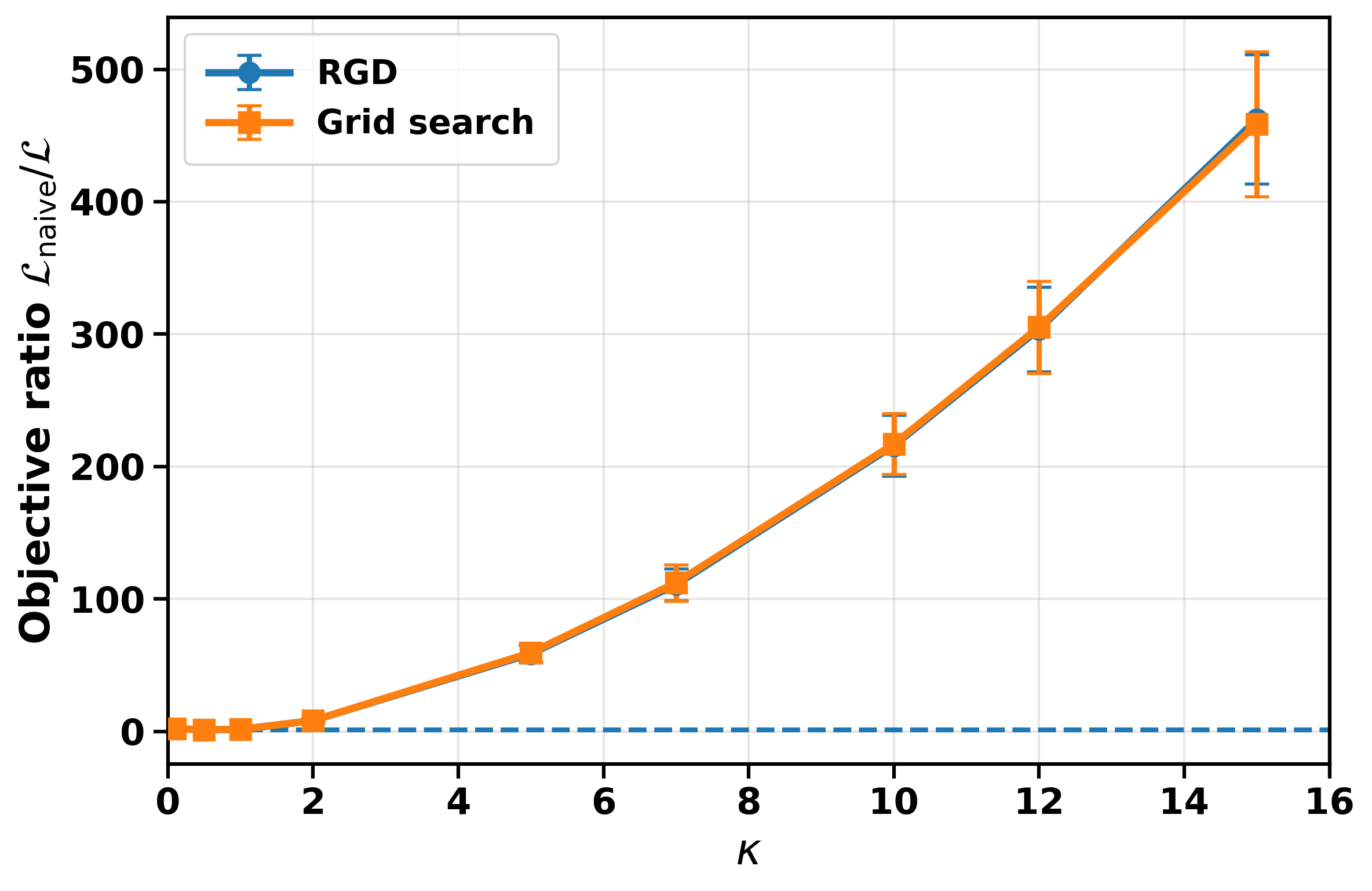}
    \end{subfigure}
    \hfill
    \begin{subfigure}[t]{0.32\textwidth}
        \centering
        \includegraphics[width=\linewidth]
        {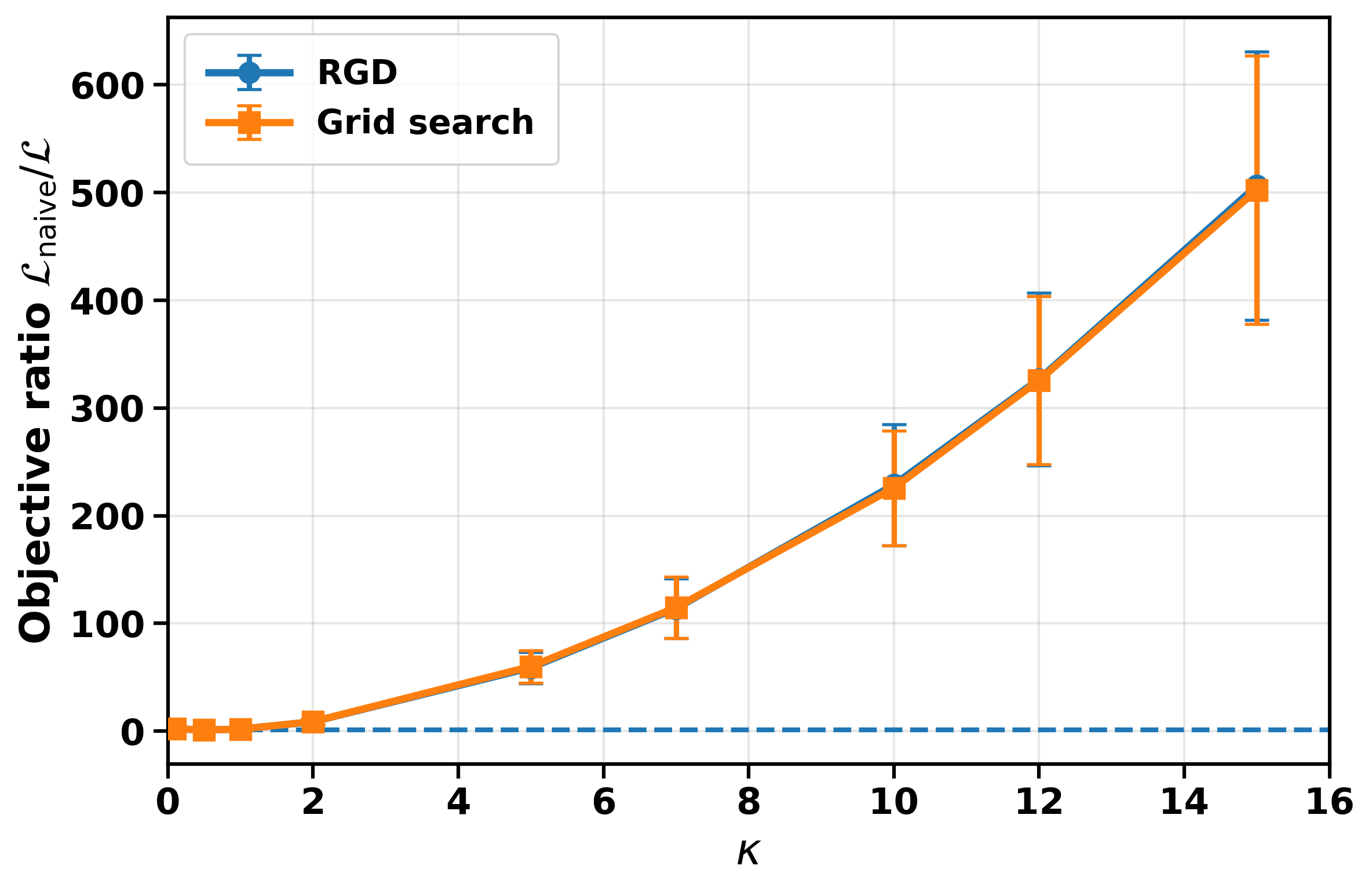}
    \end{subfigure}

    \vspace{0.6em}

    \begin{subfigure}[t]{0.32\textwidth}
        \centering
        \includegraphics[width=\linewidth]
        {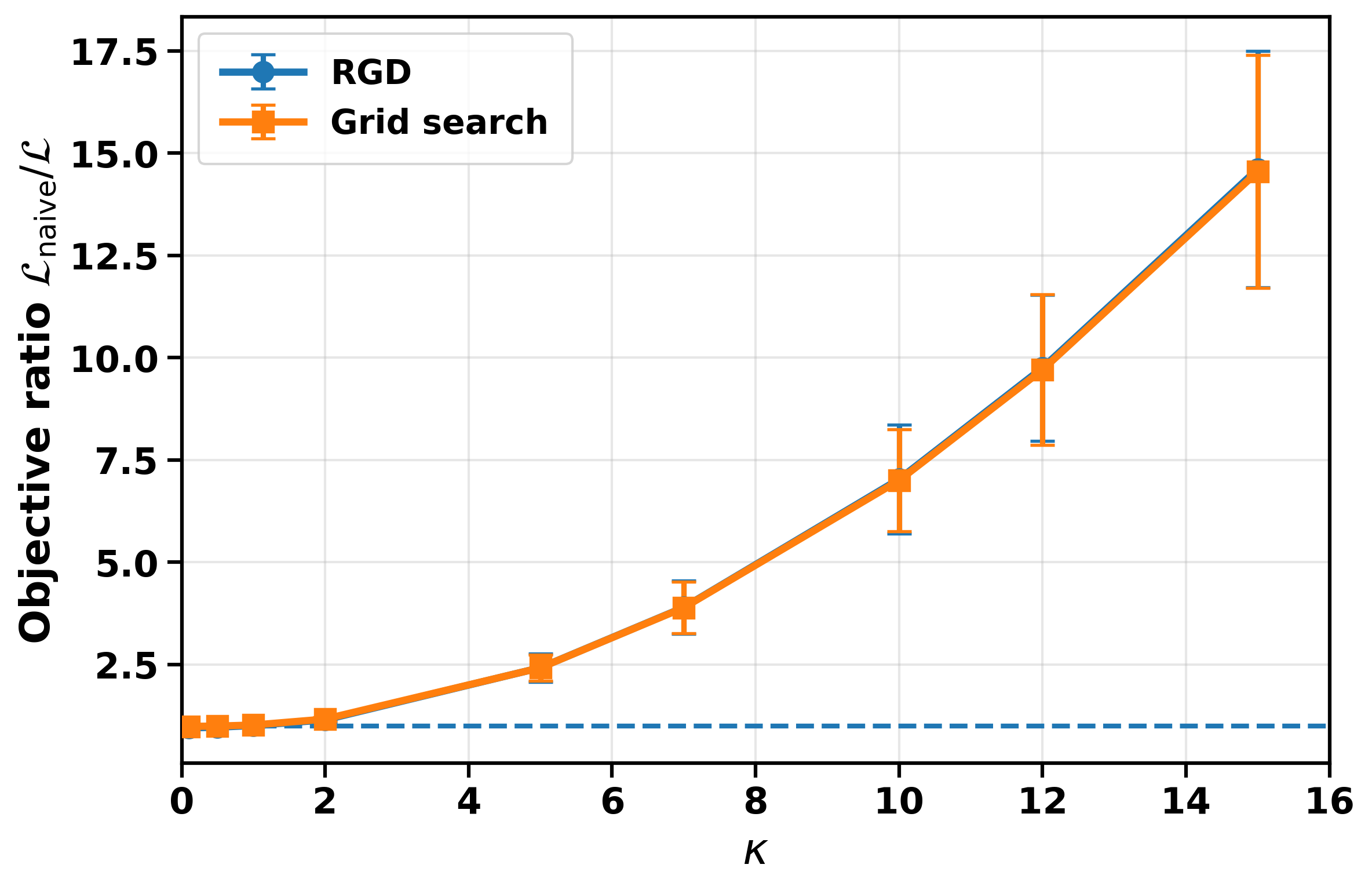}
        \caption{Uniform costs}
    \end{subfigure}
    \hfill
    \begin{subfigure}[t]{0.32\textwidth}
        \centering
        \includegraphics[width=\linewidth]
        {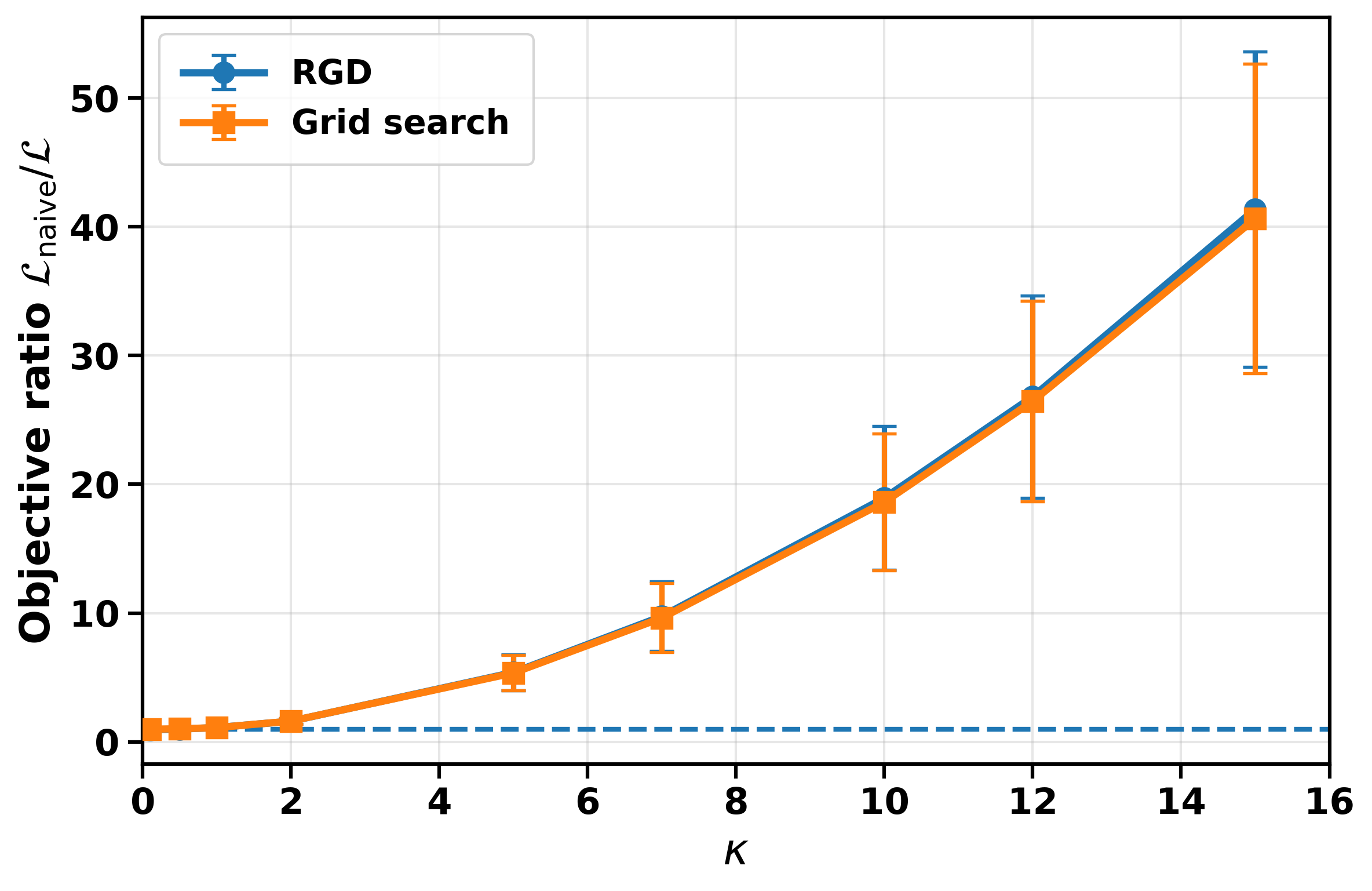}
        \caption{Upstream features cheap}
    \end{subfigure}
    \hfill
    \begin{subfigure}[t]{0.32\textwidth}
        \centering
        \includegraphics[width=\linewidth]
        {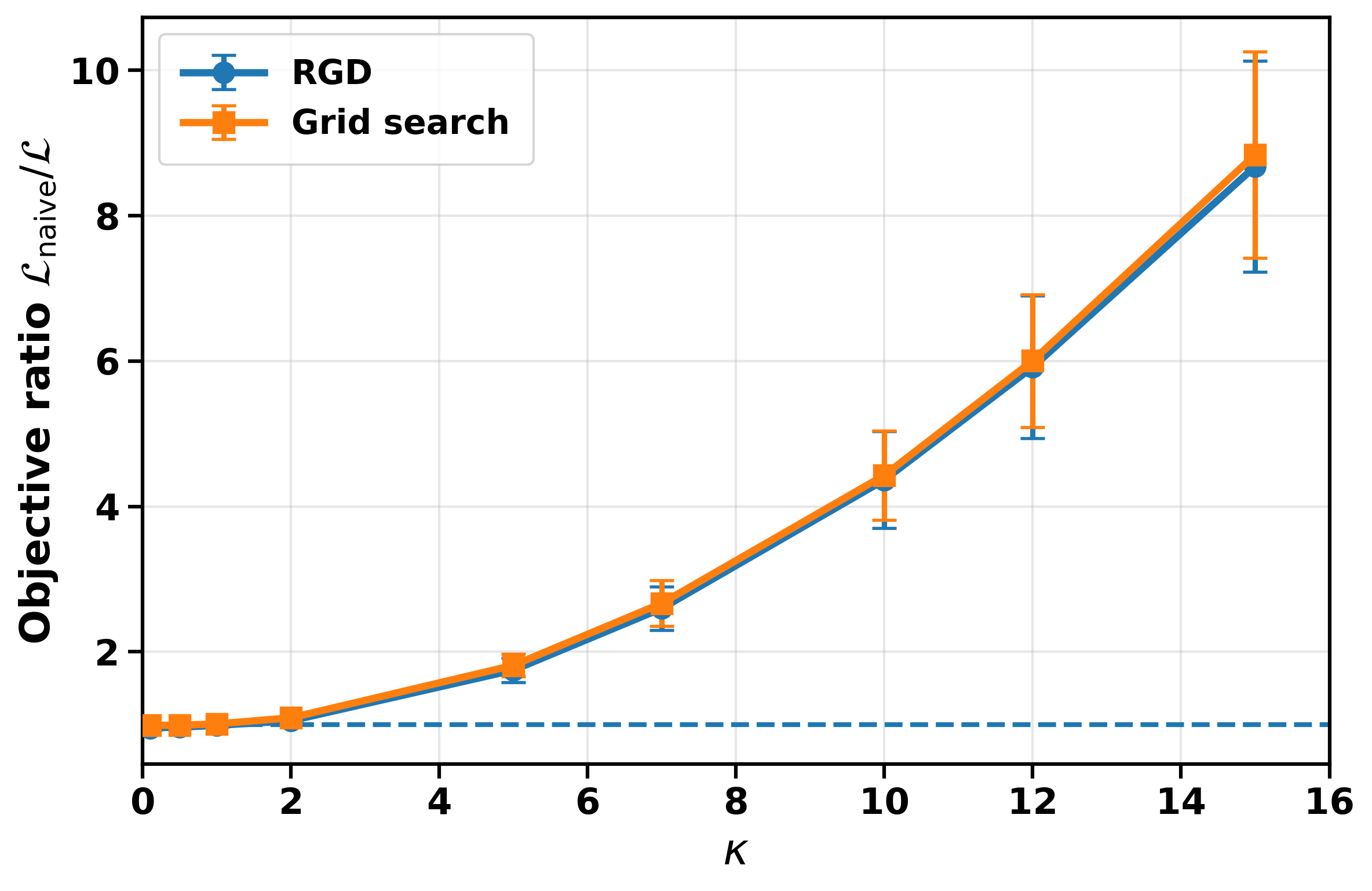}
        \caption{Downstream features cheap}
    \end{subfigure}

    \caption{
  Ratio of the performative objective under the naive model to that under the stable (RGD) and Grid-search models across three intervention cost profiles. The top row
    corresponds to the semi-synthetic dataset, while the bottom row
    corresponds to the Taiwan credit dataset. From left to right, the
    columns represent uniform costs, upstream features being cheaper to
    intervene on, and downstream features being cheaper to intervene on.
    }
    \label{fig:improvement-ratio-cost-profiles}
\end{figure*}

\begin{figure*}[t!]
    \centering

    \begin{subfigure}[t]{0.24\textwidth}
        \centering
        \includegraphics[width=\linewidth]
        {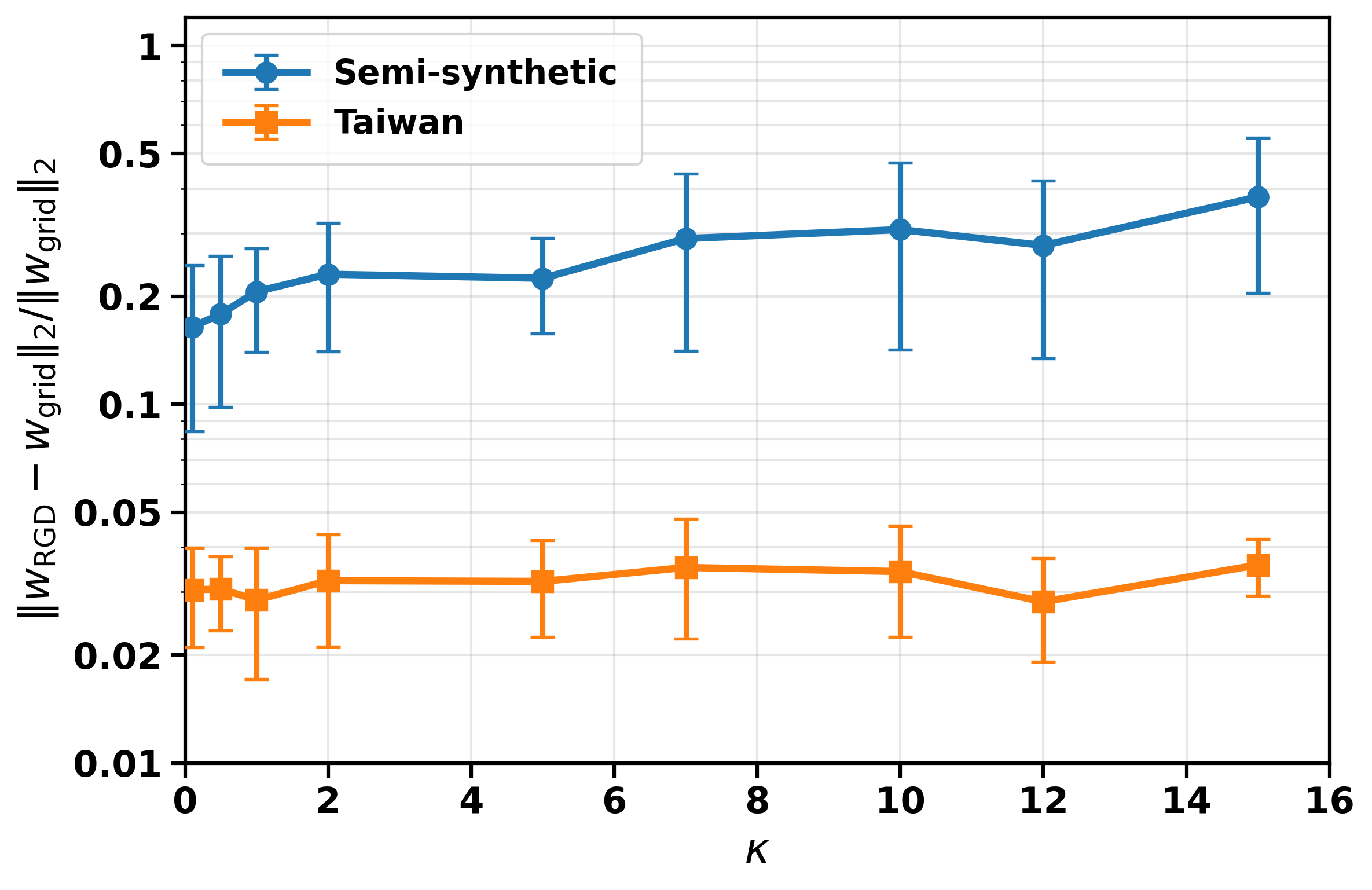}
        \caption{Default costs}
        \label{fig:rel-dist-default}
    \end{subfigure}
    \hfill
    \begin{subfigure}[t]{0.24\textwidth}
        \centering
        \includegraphics[width=\linewidth]
        {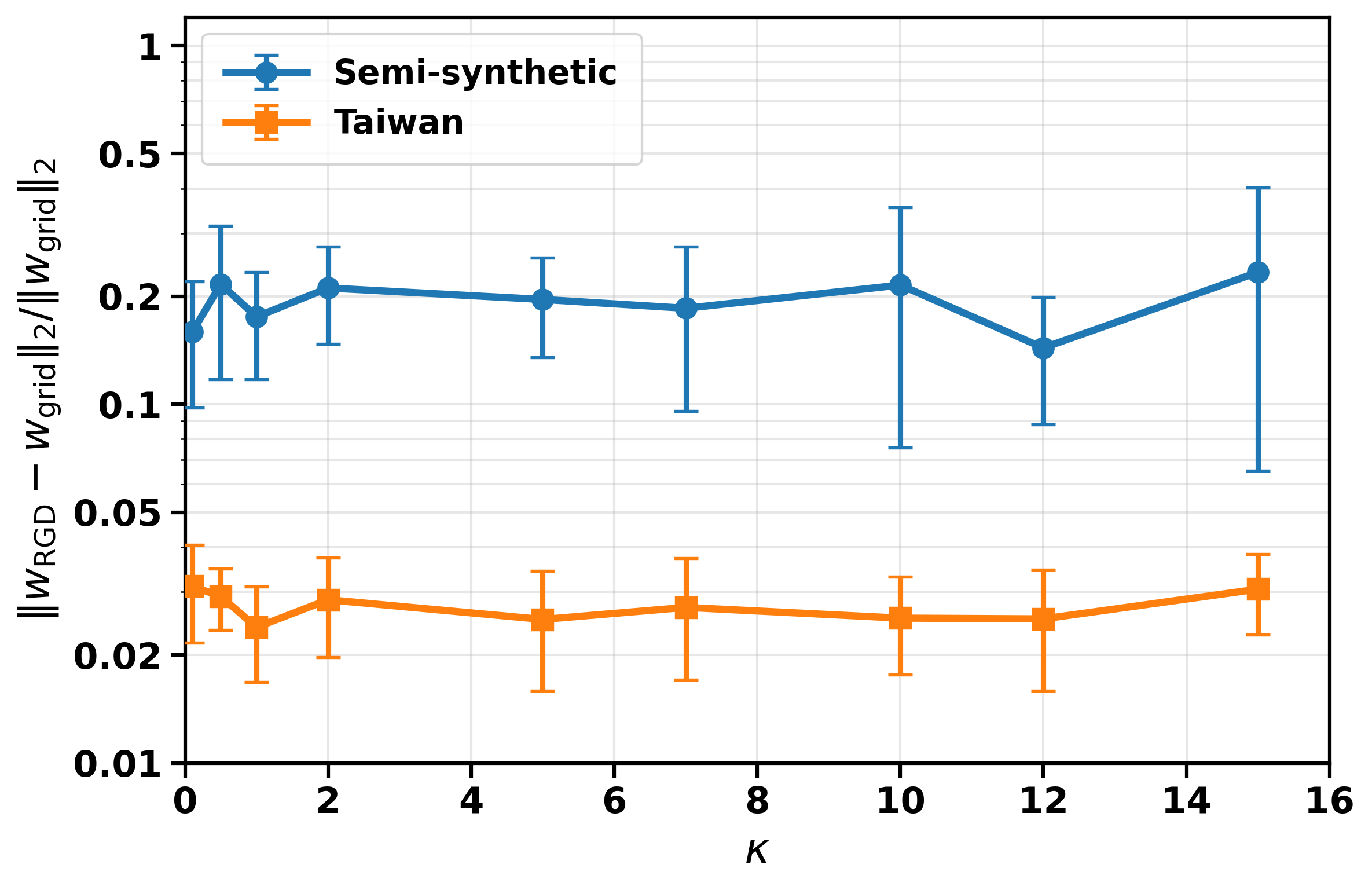}
        \caption{Uniform costs}
        \label{fig:rel-dist-uniform}
    \end{subfigure}
    \hfill
    \begin{subfigure}[t]{0.24\textwidth}
        \centering
        \includegraphics[width=\linewidth]
        {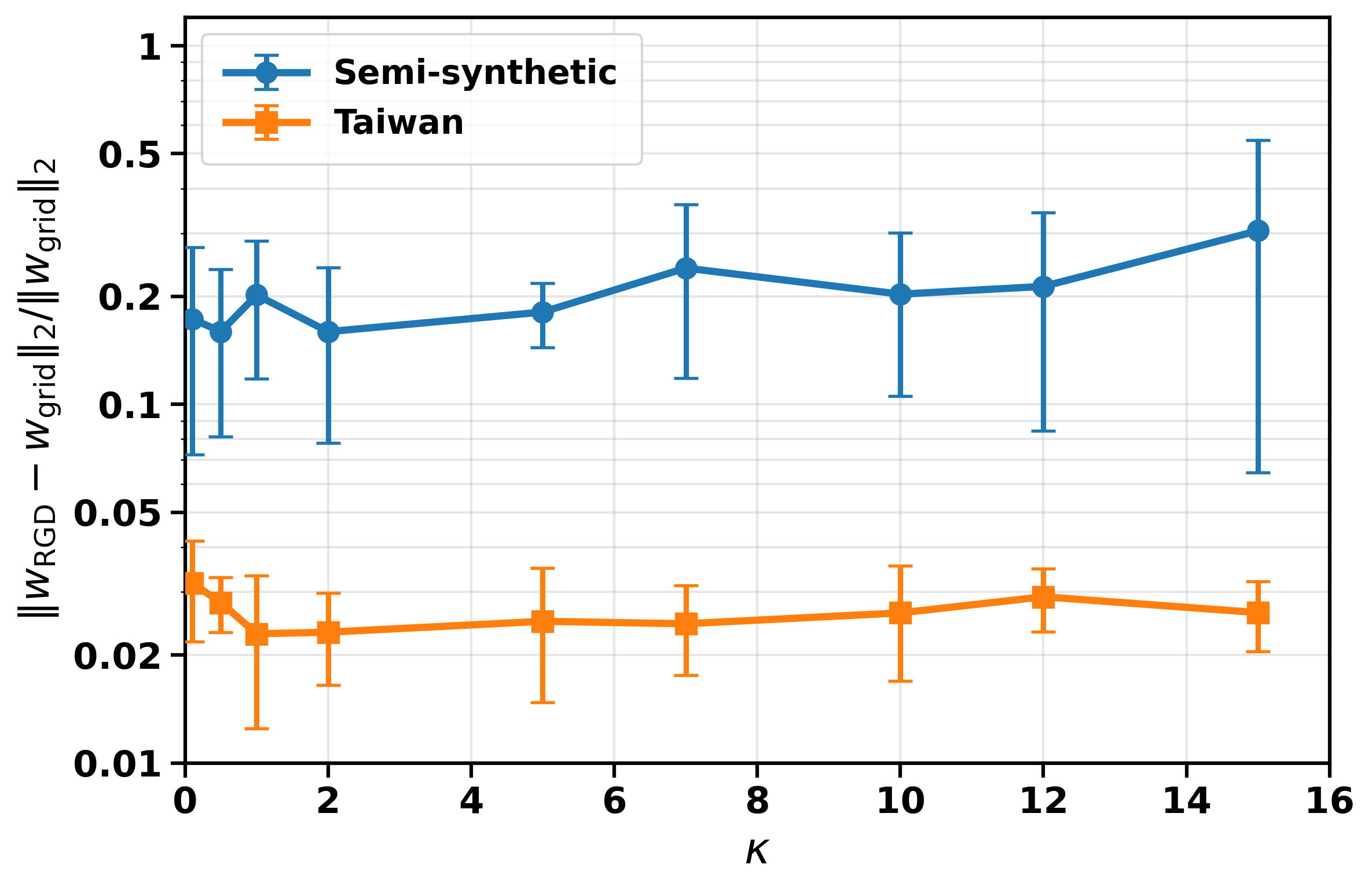}
        \caption{Upstream features cheap}
        \label{fig:rel-dist-upstream}
    \end{subfigure}
    \hfill
    \begin{subfigure}[t]{0.24\textwidth}
        \centering
        \includegraphics[width=\linewidth]
        {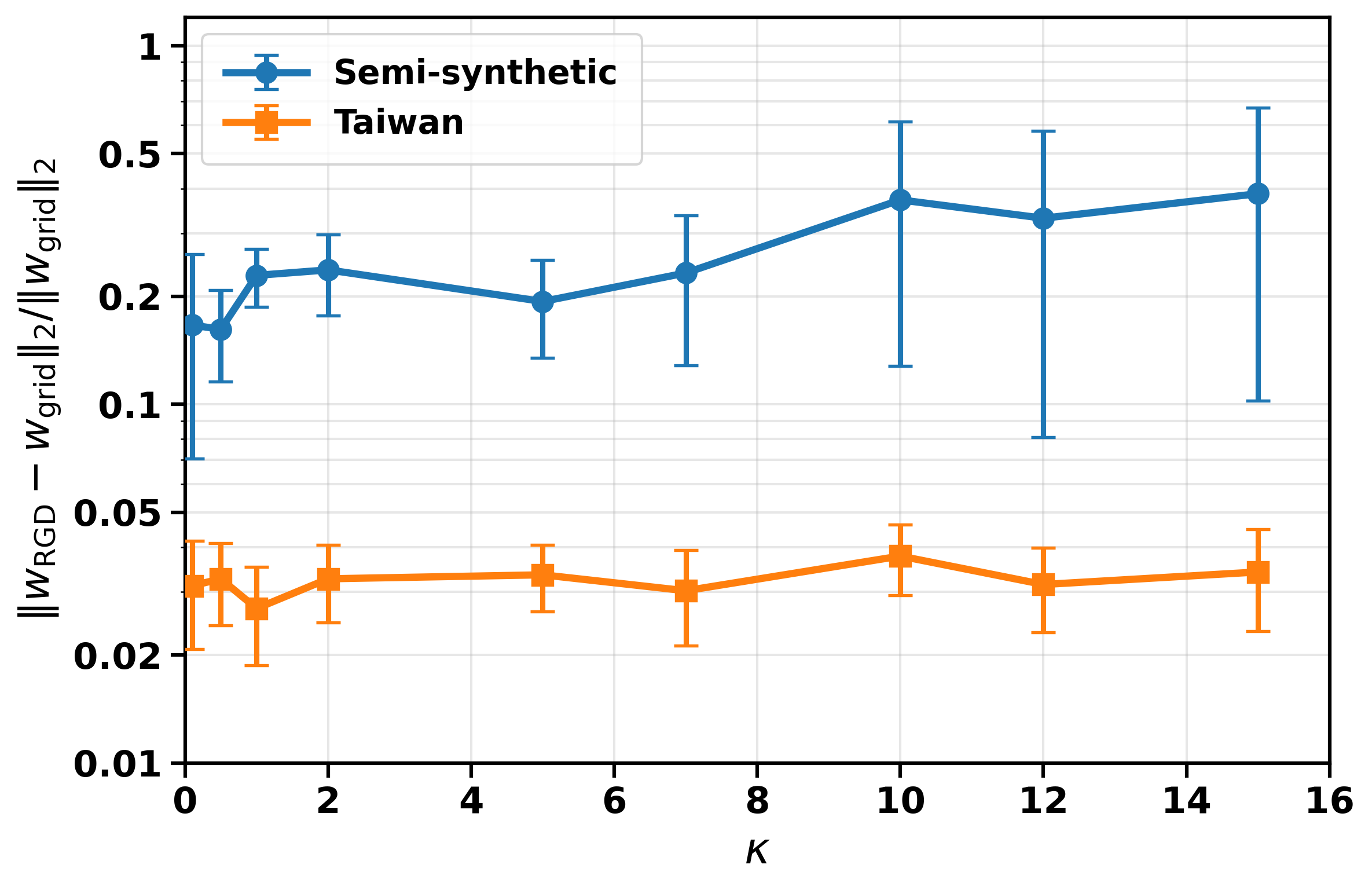}
        \caption{Downstream features cheap}
        \label{fig:rel-dist-downstream}
    \end{subfigure}

    \caption{
    Relative model-weights' distance
    $\|w_{\mathrm{RGD}}-w_{\mathrm{grid}}\|_2/\|w_{\mathrm{grid}}\|_2$
    as a function of the performativity parameter $\kappa$ under four intervention cost profiles. Each panel reports results for both the semi-synthetic and Taiwan datasets. Error bars denote one standard deviation over randomized cost matrix realizations.
    }
    \label{fig:relative-classifier-distance-cost-profiles}
\end{figure*}

\section{Additional Experimental Results}
\label{app:sensitiv}
We additionally evaluate the remaining cost profiles listed in Table~\ref{tab:cost_profiles}. Figure~\ref{fig:all-convergence} shows the convergence behavior of the proposed algorithms, Figure~\ref{fig:improvement-ratio-cost-profiles} reports the corresponding improvement ratio, and Figure~\ref{fig:relative-classifier-distance-cost-profiles} presents the relative distance between the stable (RGD) and near-optimal (Grid-search) models. Across all cost profiles, we observe behavior that is qualitatively consistent with our main experiments. In particular, the dependence of convergence, performative improvement, and the relative model distance on the performativity parameter $\kappa$ remains largely unchanged, indicating that the proposed algorithms are robust to different specifications of the action cost matrix.

\end{document}